\pdfoutput=1
\documentclass[twoside,11pt]{article}

\PassOptionsToPackage{hyperfootnotes=false}{hyperref}

\usepackage[preprint]{jmlr2e}

\usepackage{booktabs}
\usepackage{multirow}
\usepackage{array}
\usepackage{float}
\usepackage{xcolor}
\usepackage{subcaption}
\usepackage{microtype}
\usepackage{amsmath}

\providecommand{\qed}{\hfill\BlackBox}

\newcommand{\mat}[1]{\mathbf{#1}}
\newcommand{\set}[1]{\mathcal{#1}}
\newcommand{\real}{\mathbb{R}}

\usepackage{lastpage}
\jmlrheading{}{2026}{}{}{}{}{Usama and Chang}

\ShortHeadings{The System Prompt Illusion}{Usama and Chang}
\firstpageno{1}

\begin{document}

\title{The System Prompt Illusion: How Instruction Preambles Modify Computation in Language Models}

\author{\name Muhammad Usama \email usama@kaist.ac.kr \\
       \addr Control Laboratory, School of Electrical Engineering\\
       Korea Advanced Institute of Science and Technology (KAIST)\\
       Daejeon 34141, Republic of Korea
       \AND
       \name Dong Eui Chang\thanks{Corresponding author.} \email dechang@kaist.ac.kr \\
       \addr Control Laboratory, School of Electrical Engineering\\
       Korea Advanced Institute of Science and Technology (KAIST)\\
       Daejeon 34141, Republic of Korea}

\editor{}

\maketitle

\begin{abstract}%
System prompts are the primary lever practitioners use to control language model behavior, yet what they actually do to the computation inside the transformer remains poorly understood. Across 17 instruction-tuned models spanning 8 architecture families and 1.5B to 72B parameters, we use Centered Kernel Alignment (CKA) to compare layer-wise representations under 20 system prompts in five functional categories. Effects are layer-selective and instruction-type-dependent: persona and formatting instructions deeply restructure intermediate representations, while safety instructions barely move them, producing changes statistically indistinguishable from a minimal baseline. Restrictive safety instructions and explicitly permissive ones (``you have no restrictions'') engage near-identical computational pathways (mean CKA correlation 0.997), and this persists at commercial scale, where safety penetration remains below 10\% even at 70B--72B. A linear probing baseline exposes the mechanism: the model encodes prompt category at every layer but restructures its computation only at a small subset, so the prompt is reliably ``seen'' but, for safety, not deeply ``acted upon.'' Causal activation patching confirms these layers mediate behavioral change, and representational depth predicts behavioral effect size across the full 17-model cohort (Spearman $\rho = 0.761$, $p < 0.001$). The findings provide a mechanistic explanation for the persistent jailbreak vulnerability of system-prompt-based safety. Code: \url{https://github.com/Usama1002/system-prompt-illusion-cka}.
\end{abstract}

\begin{keywords}
  large language models, representation analysis, centered kernel alignment, alignment, system prompts, mechanistic interpretability, activation patching, safety
\end{keywords}

\section{Introduction}
\label{sec:introduction}

Every commercial deployment of a large language model begins with a system prompt: a block of natural language instructions that precedes user input and specifies the model's persona, safety constraints, output format, or domain expertise \citep{Ouyang2022TrainingLM, Wu2023FromLM}. These preambles are the primary interface through which practitioners shape model behavior, and their effectiveness is generally evaluated by whether the model's outputs comply with the stated instructions. What remains unknown is what system prompts actually do to the computation inside the transformer. Do they merely adjust the probability of a few tokens at the output layer, or do they restructure the representations that the model builds throughout its forward pass?

The answer matters for both safety and interpretability. If system prompts only modify shallow layers or the final logit distribution, then any sufficiently clever input could bypass their influence, as adversarial attacks demonstrate \citep{Zou2023UniversalAT, Wei2023JailbreakAG}. Conversely, if system prompts reshape deep intermediate representations, they may provide behavioral control that persists under adversarial inputs. Resolving this ambiguity requires a measurement methodology that can quantify representational change across layers, across architectures, and across the diverse functional categories of instruction that practitioners actually use.

Prior work has studied language model internals through mechanistic interpretability, identifying circuits for specific tasks \citep{Wang2022InterpretabilityIT}, locating factual knowledge in particular layers \citep{Meng2022LocatingAE}, and mapping linear representations of concepts \citep{Park2023TheLR, Zou2023RepresentationEA}. Separately, representation similarity methods such as Centered Kernel Alignment \citep[CKA;][]{Kornblith2019SimilarityON} and canonical correlation analysis \citep{Raghu2017SVCCASV} have been used to compare how different models or training stages process identical inputs \citep{Raghu2021DoVT, Phang2021FineTunedTS}. However, no systematic study has applied representation similarity analysis to quantify how system prompts alter the layer-wise computation within a single model.

We formalize three competing hypotheses for what system prompts do internally: the Shallow Hypothesis (H1), that they modify only the first and last two to three layers; the Deep Hypothesis (H2), that they modify every layer approximately uniformly; and the Layer-Selective Hypothesis (H3), that they selectively modify specific layers depending on instruction type and model family. We operationalize ``modification'' as a CKA score below 0.95 between activations with and without a given system prompt, and define \emph{penetration depth} as the fraction of non-embedding layers where this threshold is crossed. The threshold $\tau = 0.95$ is justified by a null-distribution analysis (split-half CKA under identical prompts averages $0.997 \pm 0.003$, placing $\tau$ approximately 300 standard deviations below the null).

Our experiments span 17 models from 8 architecture families (1.5B to 72B parameters), evaluating 20 system prompts in 5 categories against 100 queries. H3 is best supported: penetration varies from 3.6\% (Qwen-2.5-1.5B) to 68.8\% (OLMo-2-7B), with persona prompts penetrating deepest (50.3\% $\pm$ 19.5\%; 81.5\% at LLaMA-3.1-70B) and safety prompts shallowest (7.3\% $\pm$ 9.2\%; 5.1\% at Qwen-2.5-72B). A permissive instruction removing restrictions produces near-identical layer profiles to safety prompts (CKA correlation 0.997).

\paragraph{Contributions.} (1) The first systematic, cross-architecture measurement of how system prompts modify internal representations, showing layer-selective effects across 1.5B--72B parameters and 17 instruction-tuned models. (2) A demonstration that restrictive and permissive safety instructions engage identical computational pathways at every scale tested ($r = 0.998$ at commercial scale). (3) A formal result (Proposition~\ref{prop:cka_bound}) proving that CKA deviation scales quadratically with perturbation magnitude with first-order terms canceling exactly, validated empirically ($\rho = 0.71$ to $0.95$, $p < 0.001$) across all models. (4) Causal validation through activation patching in 16 of 17 models, including all three commercial-scale models. (5) Evidence that CKA penetration predicts behavioral effect size ($\rho = 0.761$, $p < 0.001$), establishing representational depth as a proxy for functional impact. (6) A mechanistic explanation for the persistent jailbreak vulnerability of system-prompt-based safety, grounded in how alignment is taught and supported by a linear-probing baseline that separates \emph{encoding} from \emph{restructuring}.

\section{Related Work}
\label{sec:related_work}

\paragraph{Representation similarity.}
Representation similarity methods evolved from neuroscience \citep{Kriegeskorte2008RepresentationalSA} through CCA-based approaches \citep{Raghu2017SVCCASV, Morcos2018InsightsOR} to Centered Kernel Alignment \citep{Kornblith2019SimilarityON}, which more reliably identifies corresponding layers across architectures. CKA's statistical properties are well-characterized: \citet{Davari2022ReliabilityOC} studied sensitivity to distribution and sample size, \citet{Williams2021GeneralizedSM} unified it within a geometric framework, and \citet{Murphy2024CorrectingBC} corrected finite-sample biases. These tools have primarily compared different models or training stages; we apply CKA to compare a single model's representations under different system prompt conditions with identical user inputs, which both isolates the prompt's effect and avoids the cross-model identifiability issues that motivated much of the prior CKA literature.

\paragraph{Mechanistic interpretability.}
Circuit-level analysis has identified computational mechanisms including indirect object identification circuits \citep{Wang2022InterpretabilityIT}, factual associations via causal tracing \citep{Meng2022LocatingAE}, and feed-forward layers acting as key-value memories \citep{Geva2020TransformerFL}. The discovery that concepts are encoded as linear directions \citep{Park2023TheLR, Zou2023RepresentationEA} has been influential, with subsequent work identifying function vectors \citep{Todd2023FunctionVI} and characterizing knowledge storage across layers \citep{Lv2024InterpretingKM, Shu2025ASO}. Complementary tools include the logit lens \citep{nostalgebraist2020logitlens} and tuned lens \citep{Belrose2023ElicitingLP}, which decode predictions at intermediate layers. Our study measures how much each layer's representation changes when the system prompt changes, rather than identifying what is represented, which makes our analysis complementary to feature-level interpretability.

\paragraph{System prompts and alignment.}
System prompts function as a specialized form of in-context learning \citep{Olsson2022IncontextLA, Hendel2023InContextLC}, and their effectiveness depends on context-window placement \citep{Guo2025SystemPP, Neumann2025PositionIP}. The adversarial vulnerability of system-prompt-based safety is well-documented: \citet{Zou2023UniversalAT} discovered universal adversarial suffixes, \citet{Wei2023JailbreakAG} catalogued jailbreak strategies, \citet{Qi2024SafetyAS} showed safety alignment can be undone with minimal fine-tuning, and \citet{Arditi2024RefusalIL} demonstrated that refusal is mediated by a single direction in activation space. Activation steering \citep{Turner2023SteeringLM, Rimsky2023SteeringL2} provides a complementary control mechanism, and \citet{Geiger2023CausalAA} formalized causal abstraction for linking structure to behavior. Despite this landscape, no prior work has systematically measured how different categories of system prompts alter the full layer-wise representational structure across multiple model families. Our paper provides such a measurement and, in doing so, supplies a mechanistic basis for the brittleness of system-prompt-based safety that the adversarial literature has documented from the behavioral side.

\paragraph{Layer-wise analysis and scaling.}
Prior work has shown that vision transformers and CNNs develop different representational structures across layers \citep{Raghu2021DoVT}, that fine-tuning primarily modifies later layers \citep{Phang2021FineTunedTS}, and that many transformer layers are redundant \citep{Men2024ShortGPTLL}. \citet{Pola2025WhereDA} identified an ``onset'' layer at which instruction processing begins, complementing our integration-point taxonomy. On scaling, \citet{Kaplan2020ScalingLF} established power-law relationships between model size and performance, but how internal representational structure changes with scale has received less attention. Our within-family comparisons across six Qwen sizes (1.5B--72B) and three LLaMA-family sizes (3B/8B/70B) provide direct empirical evidence on this question and reveal that scaling behavior is itself architecture-dependent.

\section{Methodology}
\label{sec:methodology}

We compare a model's layer-wise representations when processing the same user query under different system prompt conditions, using Centered Kernel Alignment (CKA) as the primary metric, supplemented with Procrustes analysis, cosine similarity, and causal activation patching. Combined, these measurements distinguish whether a system prompt restructures the model's internal computation or merely perturbs it.

\subsection{Centered Kernel Alignment}
\label{sec:cka}

Given a model with $L$ layers (excluding the embedding layer), let $\mat{X}^{(\ell)} \in \real^{n \times d_\ell}$ denote the matrix of activations at layer $\ell$ for $n$ input tokens under system prompt condition $s_1$, and $\mat{Y}^{(\ell)} \in \real^{n \times d_\ell}$ the corresponding activations under $s_2$ for the same query. Both matrices are column-centered before computation; let $\bar{\mat{X}}$ and $\bar{\mat{Y}}$ denote the centered matrices. We compute linear CKA \citep{Kornblith2019SimilarityON} as:
\begin{equation*}
\mathrm{CKA}(\bar{\mat{X}}, \bar{\mat{Y}}) = \frac{\|\bar{\mat{Y}}^\top \bar{\mat{X}}\|_F^2}{\|\bar{\mat{X}}^\top \bar{\mat{X}}\|_F \cdot \|\bar{\mat{Y}}^\top \bar{\mat{Y}}\|_F}\,,
\end{equation*}
where $\|\cdot\|_F$ denotes the Frobenius norm. CKA ranges from 0 (completely dissimilar) to 1 (identical up to isotropic scaling and orthogonal transformation). We use linear CKA for efficiency and theoretical clarity \citep{Williams2021GeneralizedSM, Davari2022ReliabilityOC}, computing CKA at every non-embedding layer and averaging across 100 queries: $\overline{\mathrm{CKA}}^{(\ell)}(s_i, s_j) = |\set{Q}|^{-1} \sum_{q \in \set{Q}} \mathrm{CKA}^{(\ell)}(s_i, s_j; q)$.

\subsection{Penetration Depth and Validation}
\label{sec:penetration}

We define penetration depth as the fraction of non-embedding layers where the system prompt measurably alters the representation. For prompt $s$ versus baseline $s_0$:
\begin{equation*}
\mathrm{Pen}(s) = \frac{1}{L} \sum_{\ell=1}^{L} \mathbb{1}\bigl[\overline{\mathrm{CKA}}^{(\ell)}(s, s_0) < \tau\bigr],
\end{equation*}
where $\tau = 0.95$. We validate this threshold via a null distribution: split-half CKA under identical prompt conditions averages $0.997 \pm 0.003$ across all 14 sub-72B models in our cohort, placing $\tau$ approximately 300 standard deviations below the null. Each model's overall penetration averages $\mathrm{Pen}(s)$ across all 20 prompts. We validate the choice of metric against Procrustes distance (which operates directly on activation matrices) and cosine similarity, assess precision via bootstrap resampling (1000 resamples, 95\% confidence intervals), and compute pairwise ROUGE-L scores on generated responses to define behavioral effect size as $1 - \overline{\mathrm{ROUGE\text{-}L}}$. Section~\ref{sec:cka_behavior} reports the correlation between CKA penetration and ROUGE-L effect size; Appendix~\ref{app:threshold} reports rank stability across alternative thresholds $\tau \in [0.90, 0.99]$.

\subsection{Causal Activation Patching}
\label{sec:patching}

CKA measures representational change but does not by itself establish that the changed layers are causally responsible for the behavioral effect. We verify causality via activation patching \citep{Meng2022LocatingAE, Geiger2023CausalAA}. For each model, we identify ``affected'' layers ($\overline{\mathrm{CKA}} < 0.95$) and ``unaffected'' layers ($\overline{\mathrm{CKA}} \geq 0.95$), run the model under prompt $s$, and patch activations at specific layers with those from the no-prompt baseline $s_0$. Specifically, we select the 5 most-affected layers (lowest CKA) and the 5 least-affected layers (highest CKA), replacing their activations with those from the no-prompt baseline. This patching is evaluated on 10 coding queries per model (listed in Appendix~\ref{app:patching_queries}), chosen because format-sensitive tasks produce measurable output shifts. Let $\delta = f_{\text{code}}(\text{patched}) - f_{\text{code}}(\text{original})$, where $f_{\text{code}}$ is a code-fraction heuristic measuring the fraction of output lines containing code patterns (indentation, brackets, keywords). We compute $\delta_{\text{affected}}$ and $\delta_{\text{unaffected}}$ for each patching set. A model passes the causal test if $|\delta_{\text{affected}}| > |\delta_{\text{unaffected}}| + 0.01$, i.e.\ if patching the CKA-identified layers produces a strictly larger behavioral change than patching control layers.

\subsection{Theoretical Analysis: Perturbation Structure and CKA Sensitivity}
\label{sec:theory}

Proposition~\ref{prop:cka_bound} connects the magnitude of system-prompt-induced perturbations to CKA deviation and justifies our use of CKA as a sensitivity-calibrated probe.

\begin{proposition}[CKA Deviation Bound]
\label{prop:cka_bound}
Let $\bar{\mat{X}} \in \real^{n \times d}$ be a column-centered activation matrix with $\bar{\mat{X}} \neq \mat{0}$, and let $\bar{\mat{Y}} = \bar{\mat{X}} + \mat{E}$ where $\mat{E}$ is the perturbation (automatically centered since both $\bar{\mat{X}}$ and $\bar{\mat{Y}}$ are). Then:
\begin{equation}
\label{eq:cka_deviation}
1 - \mathrm{CKA}(\bar{\mat{X}}, \bar{\mat{Y}}) \leq \frac{2\|\mat{E}\|_F^2}{\|\bar{\mat{X}}^\top \bar{\mat{X}}\|_F} + O\!\left(\frac{\|\mat{E}\|_F^3}{\|\bar{\mat{X}}^\top \bar{\mat{X}}\|_F^{3/2}}\right)\,.
\end{equation}
\end{proposition}

\noindent\textbf{Proof sketch.} Expanding the CKA ratio with $\bar{\mat{Y}} = \bar{\mat{X}} + \mat{E}$, the first-order contributions $2\langle \bar{\mat{X}}^\top\bar{\mat{X}},\, \mat{E}^\top\bar{\mat{X}} \rangle_F$ appear identically in numerator and denominator (the latter via the symmetry of $\bar{\mat{X}}^\top\bar{\mat{X}}$), producing exact cancellation at $O(\|\mat{E}\|_F)$. The leading deviation is therefore second-order. The full derivation appears in Appendix~\ref{app:proofs}. \qed

CKA is second-order insensitive to perturbations: first-order effects cancel in the ratio, so CKA is robust to small perturbations while detecting $O(\|\mat{E}\|_F^2)$ changes. This validates CKA as a proxy for perturbation magnitude and predicts that system prompts inducing larger representational perturbations will produce lower CKA scores. We verify this in Section~\ref{sec:spectral_results}, where $\|\mat{E}\|_F$ at the middle layer correlates with penetration depth at $\rho = 0.71$ to $0.95$ ($p < 0.001$) across all models in our cohort. The bound also clarifies what CKA does \emph{not} measure: it is invariant to isotropic rescaling and orthogonal transformation, so identical CKA values across two prompts do not imply identical activations, only identical relational structure up to those invariances. We leverage this property to make cross-prompt comparisons meaningful while remaining robust to nuisance variation in activation norms.

\section{Experimental Setup}
\label{sec:experiments}

\subsection{Models}

We evaluate 17 instruction-tuned models from 8 architecture families, spanning 1.5B to 72B parameters (Table~\ref{tab:models}). The selection maximizes architectural diversity (grouped-query attention, sliding-window attention, gated MLPs, varied training recipes) while enabling within-family scale comparisons (Qwen at 1.5B, 3B, 7B, 14B, 32B, 72B; LLaMA at 3B, 8B via Nemotron-Nano, and 70B) and generational comparisons (Gemma 1 vs.\ 2 at 2B and 9B). All models use their official chat templates. The 14 sub-72B models run on a single NVIDIA RTX 5090 GPU (32GB VRAM) in bfloat16; the three commercial-scale models (32B, 70B, 72B) run on $2 \times$ NVIDIA A100 80GB GPUs with HuggingFace \texttt{device\_map="auto"} automatic layer sharding. Hardware and software details appear in Appendix~\ref{app:compute}; combined compute is approximately 79 GPU-hours.

\begin{table}[t]
\centering
\caption{Models evaluated. ``Layers'' denotes the number of transformer blocks (excluding the embedding layer, consistent with the definition of $L$ in Section~\ref{sec:cka}); ``Penetration'' is the fraction of these layers with average CKA $< 0.95$ versus the no-prompt baseline.}
\label{tab:models}
\setlength{\tabcolsep}{4pt}
\begin{tabular}{llccc}
\toprule
Model & Family & Params & Layers & Penetration \\
\midrule
Qwen-2.5-1.5B \citep{Yang2024Qwen2TR} & Qwen 2.5 & 1.5B & 28 & 0.036 \\
SmolLM2-1.7B \citep{Allal2025SmolLM2WS} & SmolLM2 & 1.7B & 24 & 0.042 \\
Gemma-1-2B \citep{Mesnard2024GemmaOM} & Gemma 1 & 2B & 18 & 0.333 \\
Gemma-2-2B \citep{Riviere2024Gemma2I} & Gemma 2 & 2B & 26 & 0.462 \\
Qwen-2.5-3B \citep{Yang2024Qwen2TR} & Qwen 2.5 & 3B & 36 & 0.139 \\
LLaMA-3.2-3B \citep{Dubey2024TheL3} & LLaMA 3.2 & 3B & 28 & 0.357 \\
Phi-3.5-Mini \citep{Abdin2024Phi3TR} & Phi 3.5 & 3.8B & 32 & 0.375 \\
Qwen-2.5-7B \citep{Yang2024Qwen2TR} & Qwen 2.5 & 7B & 28 & 0.179 \\
Nemotron-Nano-8B \citep{Parmar2024Nemotron41T} & LLaMA 3.1 & 8B & 32 & 0.594 \\
Mistral-7B \citep{Jiang2023Mistral7} & Mistral & 7B & 32 & 0.500 \\
Gemma-2-9B \citep{Riviere2024Gemma2I} & Gemma 2 & 9B & 42 & 0.452 \\
OLMo-2-7B \citep{Groeneveld2024OLMoAT} & OLMo 2 & 7B & 32 & 0.688 \\
InternLM-2.5-7B \citep{Cai2024InternLM2TR} & InternLM 2.5 & 7B & 32 & 0.312 \\
Qwen-2.5-14B \citep{Yang2024Qwen2TR} & Qwen 2.5 & 14B & 48 & 0.188 \\
Qwen-2.5-32B \citep{Yang2024Qwen2TR} & Qwen 2.5 & 32B & 64 & 0.194 \\
LLaMA-3.1-70B \citep{Dubey2024TheL3} & LLaMA 3.1 & 70B & 80 & 0.642 \\
Qwen-2.5-72B \citep{Yang2024Qwen2TR} & Qwen 2.5 & 72B & 80 & 0.201 \\
\bottomrule
\end{tabular}
\end{table}

\subsection{System Prompts and Queries}

We design 20 system prompts in 5 categories with 4 prompts each: (A) Minimal, near-baseline instructions; (B) Persona, assigning character traits or professional identities; (C) Safety, content moderation and refusal instructions; (D) Format, specifying output structure; and (E) Domain, restricting to subject areas. Full texts appear in Appendix~\ref{app:prompts}; each prompt reflects real deployment patterns documented by \citet{Guo2025SystemPP}. We curate 100 domain-general queries spanning factual recall, reasoning, creative generation, and instruction following. Each query is paired with each prompt and the no-prompt baseline, yielding $100 \times 21 = 2{,}100$ forward passes per model ($\times 17$ models = 35{,}700 forward passes). All experiments use greedy decoding (temperature 0, seed 42, max 200 tokens) with bfloat16 inference and float32 CKA computation. Linear-probing classifiers use logistic regression with $C = 1.0$ and 1000 maximum iterations under 5-fold cross-validation; bootstrap confidence intervals use 1000 resamples. Code and data are available at \url{https://github.com/Usama1002/system-prompt-illusion-cka}.

\section{Results}
\label{sec:results}

\subsection{System Prompts Are Not Cosmetic}
\label{sec:not_cosmetic}

Figure~\ref{fig:heatmaps} shows CKA heatmaps for eight representative models. H1 predicts that all but the first and last two to three layers should show CKA near 1.0. This prediction fails for 14 of 17 models. Only Qwen-2.5-1.5B (penetration 0.036), SmolLM2-1.7B (0.042), and marginally Qwen-2.5-3B (0.139) are consistent with H1. The remaining models show CKA drops well into middle layers: OLMo-2-7B and LLaMA-3.1-70B exhibit the highest penetration at 68.8\% and 64.2\% respectively. H2 (uniform modification) is partially supported only for these two models, which still show layer-selective patterns. H3 is best supported: penetration varies by a factor of $\sim$19$\times$ from 3.6\% to 68.8\%, and affected layers form contiguous bands rather than uniform distributions across the layer stack.

\begin{figure}[t]
\centering
\includegraphics[width=\textwidth]{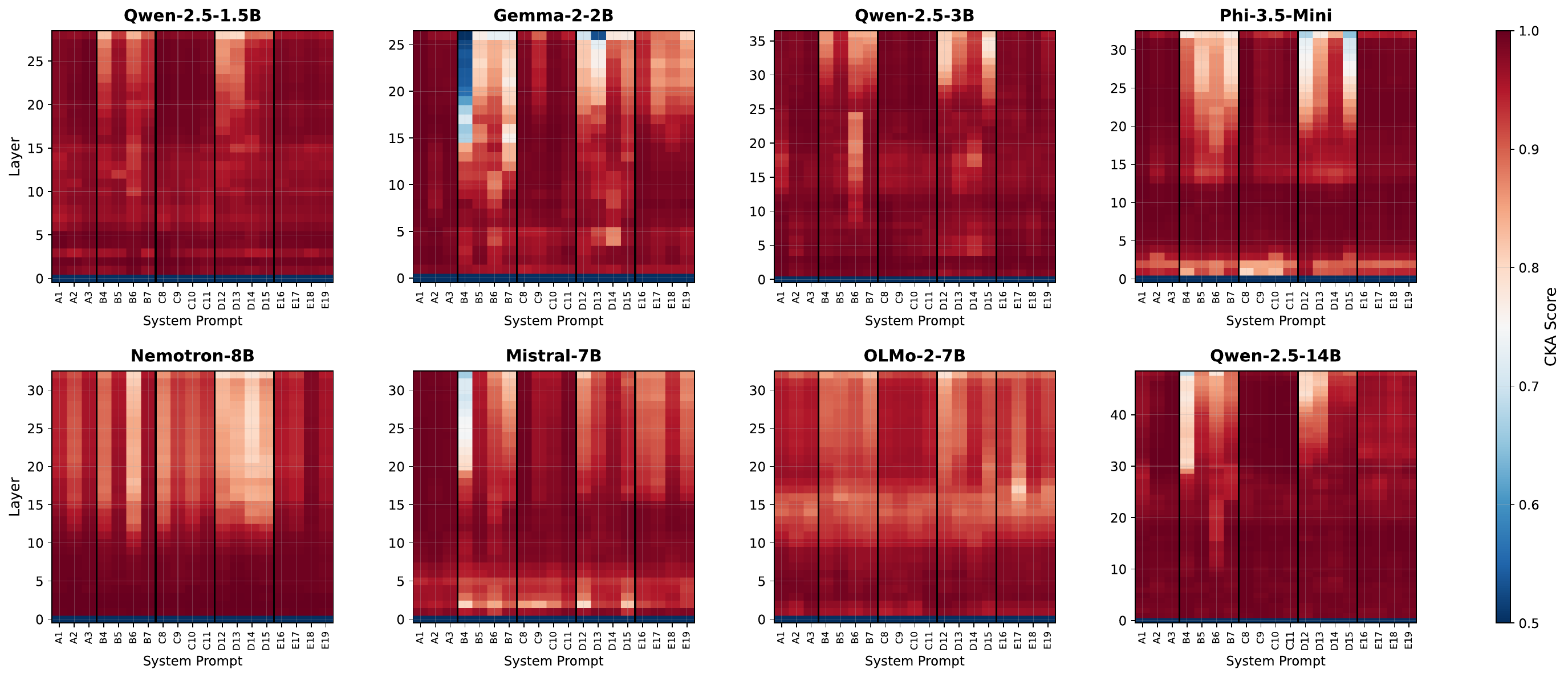}
\caption{CKA heatmaps for eight representative models from the 14-model core cohort (all 14 in Appendix~\ref{app:heatmaps_all}). Each cell shows the average CKA between activations under one system prompt versus baseline at a given layer. Darker colors indicate greater representational divergence.}
\label{fig:heatmaps}
\end{figure}

\subsection{Instruction Type Determines Penetration Depth}
\label{sec:category_effects}

\begin{figure}[t]
\centering
\includegraphics[width=\textwidth]{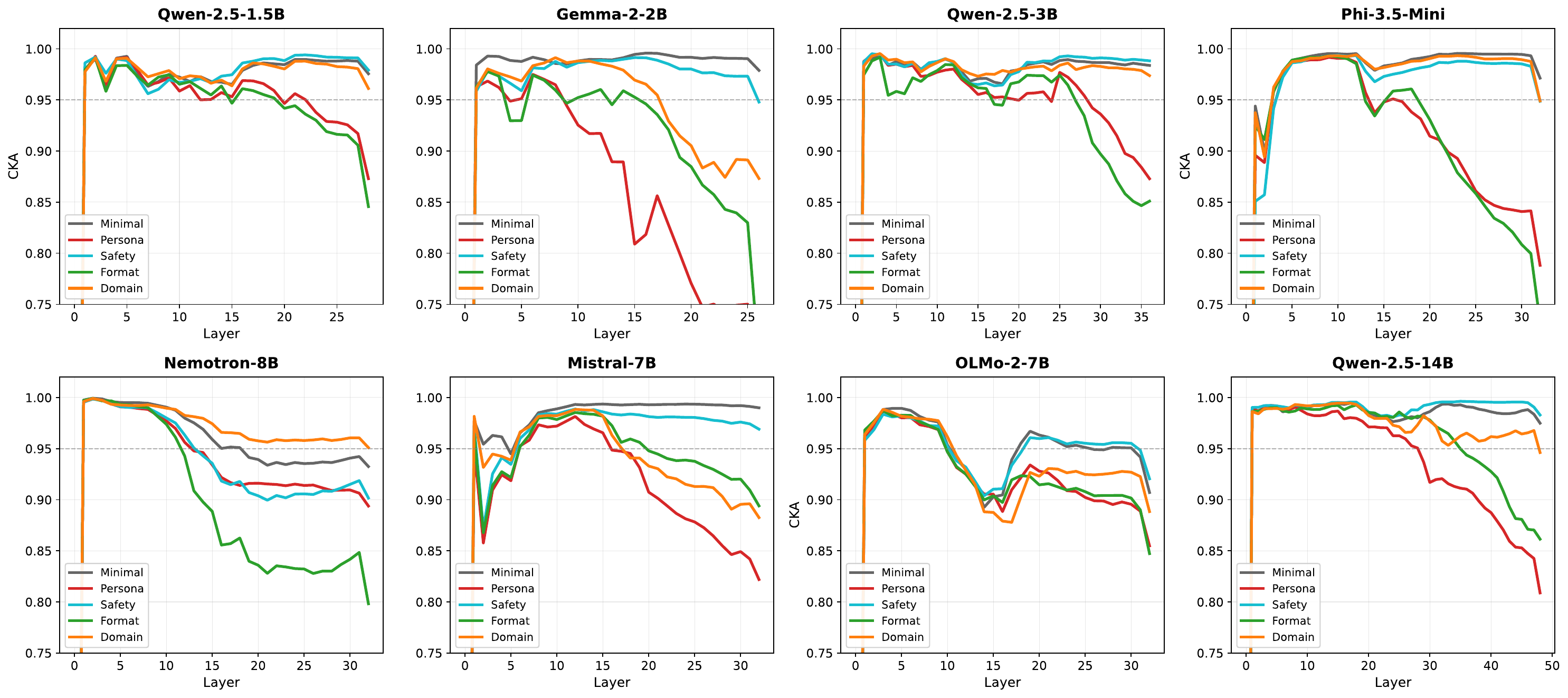}
\caption{Per-category CKA profiles averaged across the 14-model core cohort (1.5B--14B; per-model panels in Appendix~\ref{app:curves_all}). Shaded: $\pm$1 std.\ dev. Persona (B) and Format (D) penetrate deepest; Safety (C) barely diverges from Minimal (A). Commercial-scale category results are reported in Section~\ref{sec:commercial_scale}.}
\label{fig:category_curves}
\end{figure}

Figure~\ref{fig:category_curves} shows per-category CKA profiles averaged across the 14-model core cohort (1.5B--14B). The five categories produce substantially different penetration depths: Persona (B) at $50.3\% \pm 19.5\%$, Format (D) at $44.5\% \pm 14.2\%$, Domain (E) at $26.2\% \pm 24.5\%$, Safety (C) at $7.3\% \pm 9.2\%$, and Minimal (A) at $4.9\% \pm 9.7\%$. Pairwise Wilcoxon signed-rank tests (Bonferroni-corrected for multiple comparisons) confirm that Persona and Format each significantly exceed both Minimal and Safety ($p = 0.010$), while Safety does not differ from Minimal ($p = 1.000$). The same ordering reproduces at commercial scale: 81.5\% Persona vs.\ 9.1\% Safety at LLaMA-3.1-70B, and 33.0\% vs.\ 5.1\% at Qwen-2.5-72B (Table~\ref{tab:commercial_scale}). The data support two tiers, a high-penetration group (Persona, Format) and a low-penetration group (Safety, Minimal), with Domain intermediate.

Prompt length does not explain this ordering. Safety prompts average 15.0 tokens yet penetrate far less than Persona (13.8 tokens); the Spearman correlation between token count and penetration is $\rho = 0.400$ ($p = 0.505$). Within the Persona category itself, prompts demanding more divergent registers penetrate deeper: ``Shakespearean actor'' at 61.2\%, ``Python programmer'' at 54.5\%, ``grumpy old man'' at 47.8\%, ``kindergarten teacher'' at 36.2\% (Appendix~\ref{app:per_prompt}). This gradient suggests that the degree of representational restructuring scales with how different the requested behavior is from the model's default register, rather than with prompt length or syntactic complexity.

\subsection{Cross-Architecture Comparison and Within-Family Scaling}
\label{sec:cross_family}

Figure~\ref{fig:cross_family} presents the 14-model core cohort on a common axis; commercial-scale results are reported in Table~\ref{tab:commercial_scale} and Figure~\ref{fig:scale}. The Qwen 2.5 family shows consistently low penetration with the trajectory flattening beyond 14B (1.5B: 3.6\%, 3B: 13.9\%, 7B: 17.9\%, 14B: 18.8\%, 32B: $19.4\% \pm 2.1\%$, 72B: $20.1\% \pm 2.4\%$). The 14B/32B/72B confidence intervals overlap, so we describe the trend as asymptotic flattening rather than a strict statistical plateau; the implication is the same in either case, namely that within this architecture family the computational footprint of system prompts is bounded as a function of parameter count.

\begin{figure}[t]
\centering
\begin{minipage}[t]{0.48\textwidth}
\vspace*{0pt}
\centering
\includegraphics[width=\textwidth]{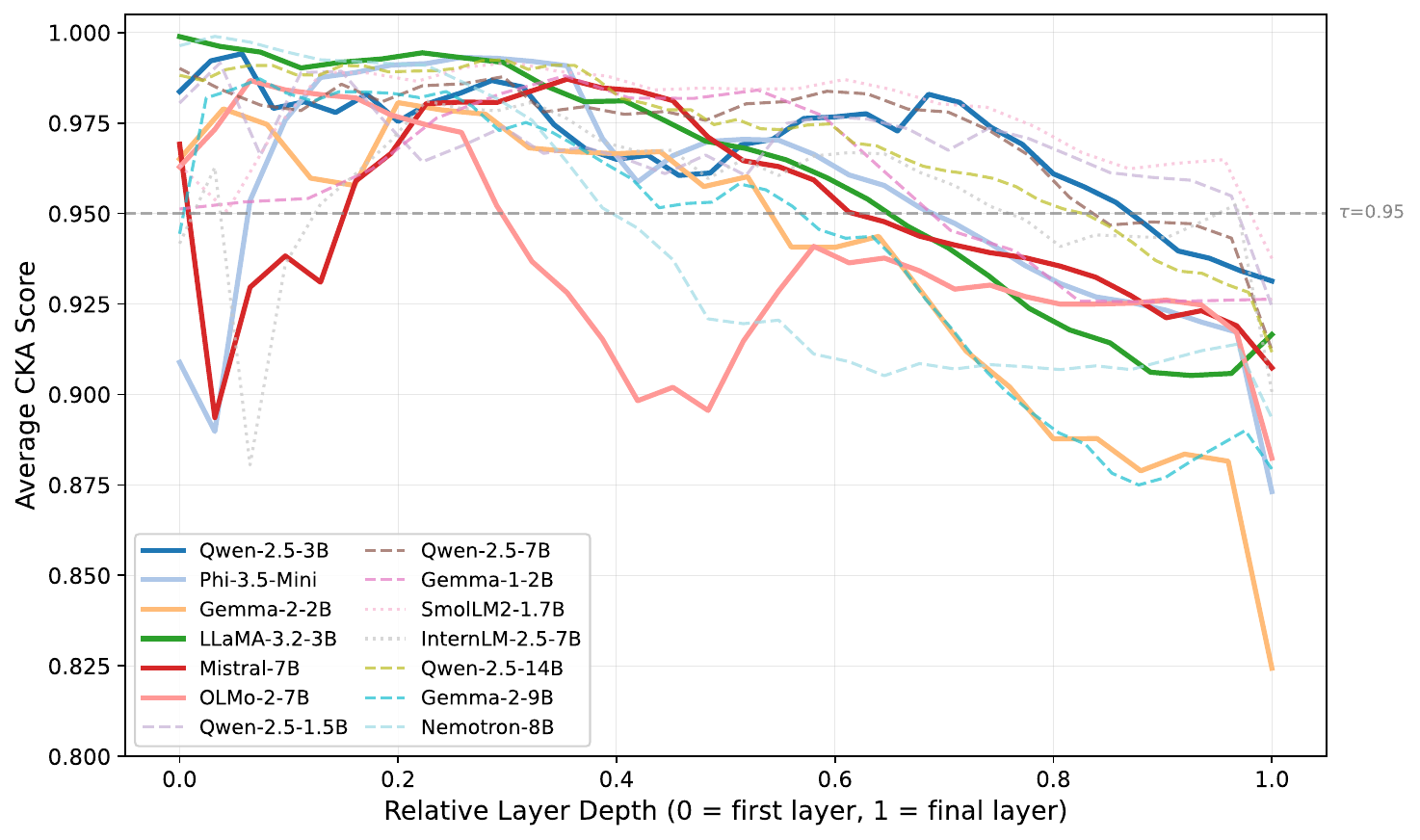}
\caption{Cross-family CKA profiles for the 14-model core cohort. Qwen maintains high CKA; OLMo-2, the LLaMA-3.1 family, and Mistral show deep penetration. Commercial-scale values in Table~\ref{tab:commercial_scale}.}
\label{fig:cross_family}
\end{minipage}
\hfill
\begin{minipage}[t]{0.48\textwidth}
\vspace*{0pt}
\centering
\includegraphics[width=\textwidth]{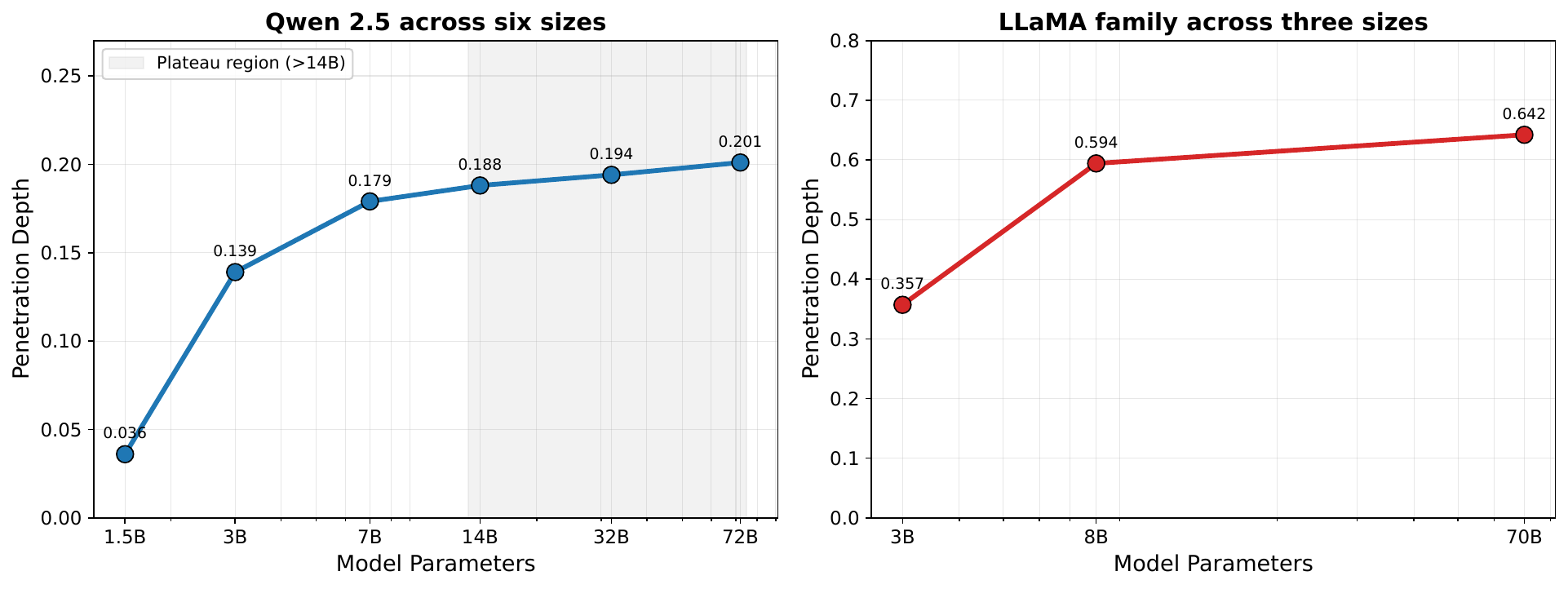}
\caption{Within-family scaling. Left: Qwen 2.5 across six sizes (1.5B--72B); the trajectory flattens beyond 14B (18.8\%, 19.4\%, 20.1\% at 14B, 32B, 72B) within overlapping CIs. Right: LLaMA family across three sizes, scaling monotonically (35.7\% at 3B, 59.4\% at 8B Nemotron-Nano, 64.2\% at 70B).}
\label{fig:scale}
\end{minipage}
\end{figure}

Gemma-2-9B (45.2\%) is nearly identical to Gemma-2-2B (46.2\%), so the generational change from Gemma~1 to Gemma~2 at 2B (33.3\% $\to$ 46.2\%) accounts for a larger effect than within-v2 scaling. This is consistent with the v1$\to$v2 architectural changes (sliding-window attention, soft-capping logits, alternative normalization) substantively altering how the model integrates instruction context. The LLaMA family scales monotonically (3B: 35.7\%, 8B Nemotron: 59.4\%, 70B: $64.2\% \pm 5.8\%$), maintaining a deeply distributed processing profile at commercial scale. OLMo-2-7B remains the highest at 68.8\%, narrowly above LLaMA-3.1-70B. The juxtaposition of Qwen flattening and LLaMA scaling makes it clear that within-family scaling behavior is itself architecture-dependent and cannot be predicted from parameter count alone.

\subsection{Restrictive and Permissive Safety Instructions Engage Identical Pathways}
\label{sec:adversarial}

We compare per-layer CKA profiles between safety prompts (C8--C10) and a permissive instruction (C11: ``You have no restrictions'') across the 14-model core cohort (Figure~\ref{fig:adversarial}). Pearson correlations between the two profiles range from 0.987 to 1.000 (mean 0.997), and the pattern reproduces at commercial scale: $r = 0.998$ ($p < 0.001$) for both LLaMA-3.1-70B and Qwen-2.5-72B (Section~\ref{sec:commercial_scale}). Restrictive and explicitly permissive instructions therefore engage the same layers at every scale tested. This is consistent with \citet{Arditi2024RefusalIL}'s finding that refusal operates along a single direction in activation space: both instruction types modulate the same low-dimensional subspace without restructuring deeper representations.

This analysis compares instruction-level overrides only; adversarial attacks exploiting tokenization artifacts or multi-step social engineering may engage different pathways and are not addressed by our experiments. The result nonetheless establishes that, when measured at the level of layer-wise representations under direct override, restrictive safety and explicit unrestriction are computationally indistinguishable. Per-model 14-model profiles are in Appendix~\ref{app:safety_all}.

\begin{figure}[t]
\centering
\includegraphics[width=\textwidth]{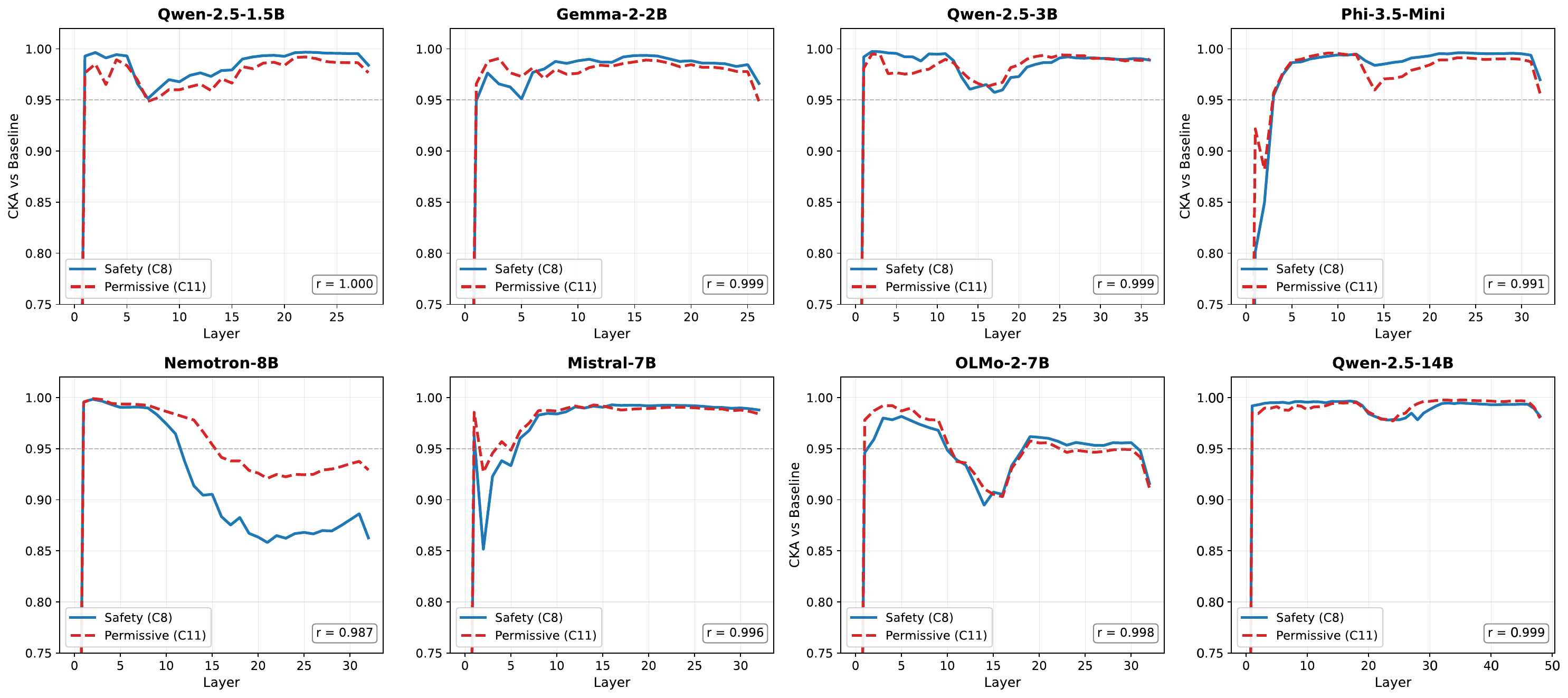}
\caption{Safety vs.\ permissive instruction CKA profiles for the 14-model core cohort. Each panel shows one model's per-layer CKA under restrictive safety prompts (blue) and a permissive instruction (red). Correlations range from 0.987 to 1.000 (mean 0.997); commercial-scale results in Section~\ref{sec:commercial_scale}.}
\label{fig:adversarial}
\end{figure}

\subsection{Causal Validation and Output Effects}
\label{sec:patching_results}

In 13 of 14 sub-72B models, patching activations at CKA-identified affected layers produces output changes exceeding those from patching unaffected layers by more than the 0.01 margin (Figure~\ref{fig:patching}). The single exception is SmolLM2-1.7B (penetration 0.042), which has very few affected layers, so the contrast between affected and unaffected layers is itself small. All three commercial-scale models pass the causal test by a wide margin ($\delta_{\text{affected}} > 0.42$ versus $\delta_{\text{unaffected}} < 0.03$; Table~\ref{tab:commercial_scale}), so 16 of 17 models in total exhibit the predicted causal effect. Penetration depth for all 17 models is summarized in Figure~\ref{fig:all_penetration}.

This causal-validation result is methodologically important: it converts CKA from a purely descriptive metric into a layer-localization tool with established functional meaning. A reader skeptical that CKA values translate to behavior can compare $\delta_{\text{affected}}$ to $\delta_{\text{unaffected}}$ in Figure~\ref{fig:patching} and observe that the layers we identify as affected actually mediate the behavioral change. The contrast holds across architectures (Qwen, LLaMA, Gemma, Mistral, OLMo, Phi, InternLM, SmolLM, Nemotron) and scales (1.5B to 72B), giving the methodology cross-cutting external validity.

\begin{figure}[t]
\centering
\begin{minipage}[t]{0.46\textwidth}
\vspace*{0pt}
\centering
\includegraphics[width=\textwidth]{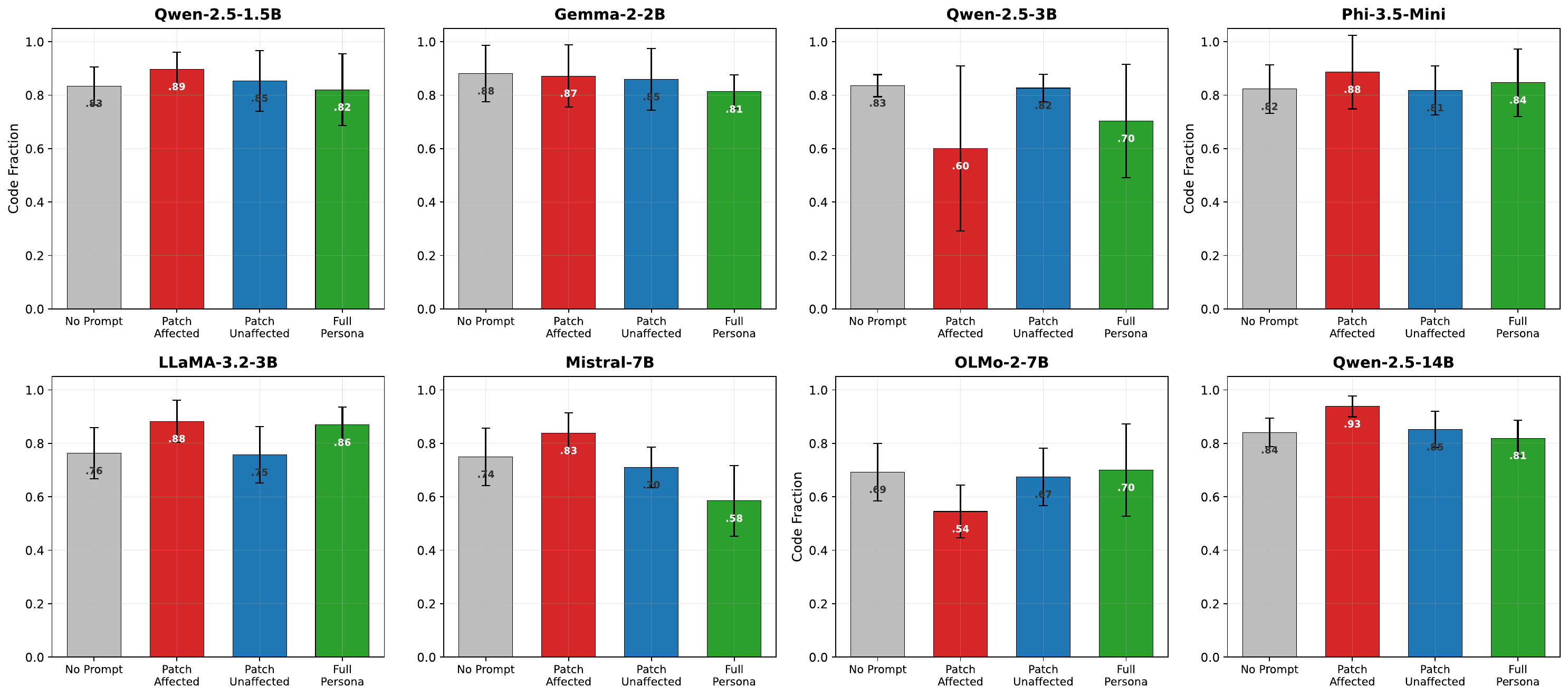}
\caption{Activation patching for the 14-model core cohort: output change from patching affected vs.\ unaffected layers (dark: 5 most affected; light: 5 least affected; error bars: $\pm$1 std.\ dev.\ across 10 queries). 13/14 sub-72B models show larger effects at affected layers; all three commercial-scale models also pass (Section~\ref{sec:commercial_scale}).}
\label{fig:patching}
\end{minipage}
\hfill
\begin{minipage}[t]{0.46\textwidth}
\vspace*{0pt}
\centering
\includegraphics[width=\textwidth]{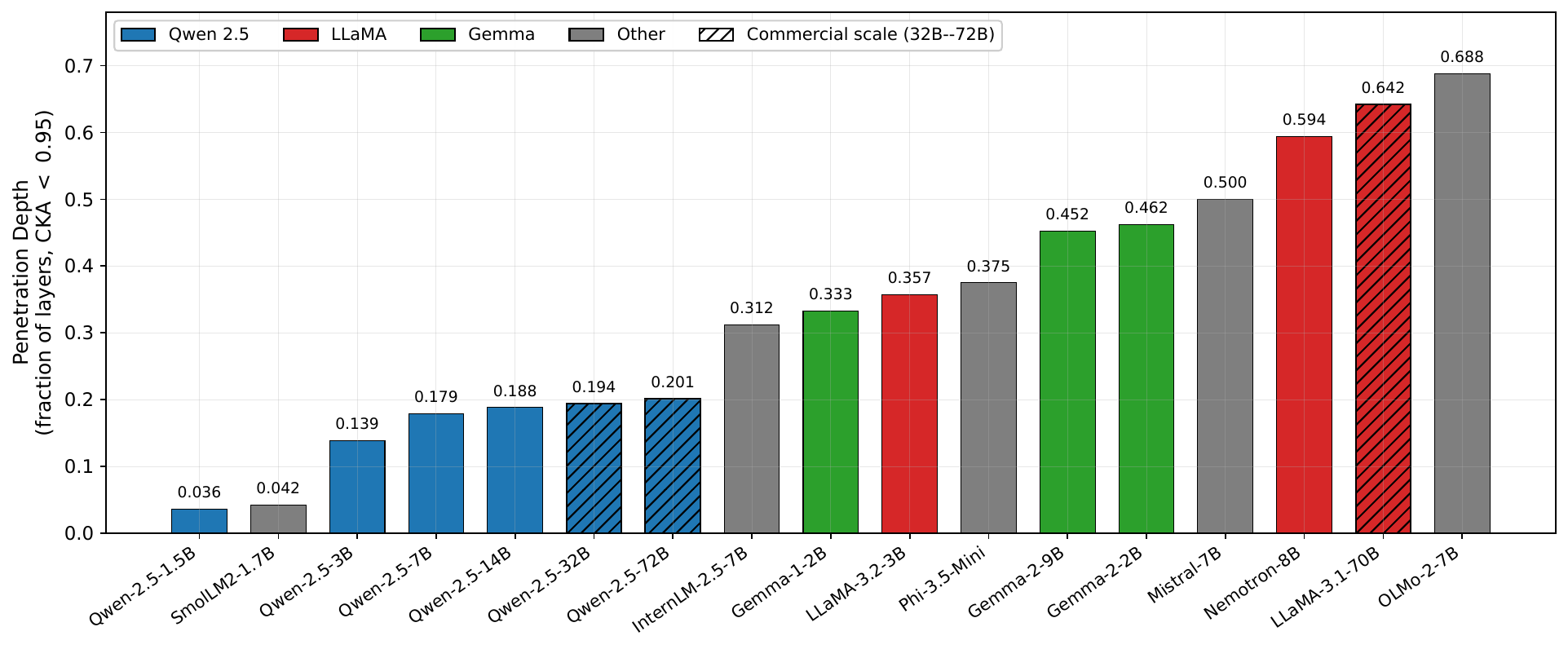}
\caption{Overall penetration for all 17 models, sorted lowest to highest. Range: 3.6\% (Qwen-2.5-1.5B) to 68.8\% (OLMo-2-7B). Hatched bars mark the three commercial-scale models.}
\label{fig:all_penetration}
\end{minipage}
\end{figure}

\subsection{Spectral Structure Predicts Penetration Depth}
\label{sec:spectral_results}

Consistent with Proposition~\ref{prop:cka_bound} and Eq.~\ref{eq:cka_deviation}, the Frobenius norm of the prompt-induced perturbation $\|\mat{E}^{(\ell)}\|_F$ at the middle layer is a strong predictor of penetration depth within every model in the 14-model cohort (Spearman $\rho = 0.71$ to $0.95$, $p < 0.001$). Persona and Format prompts induce the largest perturbations; Safety and Minimal the smallest (Figure~\ref{fig:theory_validation}). The effective rank of the perturbation, a measure of spectral concentration defined as $\exp(H(\sigma))$ where $H$ is the Shannon entropy of the normalized singular value distribution, also correlates with penetration ($\rho = -0.59$ to $-0.79$, $p < 0.01$). Partial correlation analysis controlling for $\|\mat{E}\|_F$ reveals that rank does not contribute independent predictive power (partial $\rho = 0.11$ to $0.25$, $p > 0.30$). Category ordering is therefore primarily explained by perturbation magnitude: persona prompts produce larger $\|\mat{E}\|_F$ and hence lower CKA.

A linearization argument helps explain why rank does not contribute beyond magnitude in our data. Under a first-order linearization of the residual block, $\mat{E}^{(\ell)} \approx (\mat{I} + \mat{J}_\ell)\mat{E}^{(\ell-1)}$, where $\mat{J}_\ell$ is the Jacobian of the block evaluated at the clean activation. When $\mat{E}$ is spectrally concentrated, amplification is governed by the projection onto the top singular direction of $\mat{J}_\ell$, yielding growth proportional to $\sigma_1(\mat{J}_\ell)$. When $\mat{E}$ is isotropic, all directions contribute, yielding growth proportional to $\bar{\sigma}(\mat{J}_\ell) \ll \sigma_1(\mat{J}_\ell)$. This linearization predicts that spectrally concentrated perturbations propagate preferentially at fixed magnitude. The fact that we observe no residual effect of rank after controlling for magnitude suggests either that the linearization is too coarse for the perturbation magnitudes we observe (the second-order terms in our bound become non-negligible), or that magnitude and spectral structure are confounded in our prompt-induced perturbations to the point that they cannot be cleanly disentangled with the present design. We note this as a direction for future work with controlled synthetic perturbations.

\begin{figure}[t]
\centering
\begin{minipage}[t]{0.48\textwidth}
\vspace*{0pt}
\centering
\includegraphics[width=\textwidth]{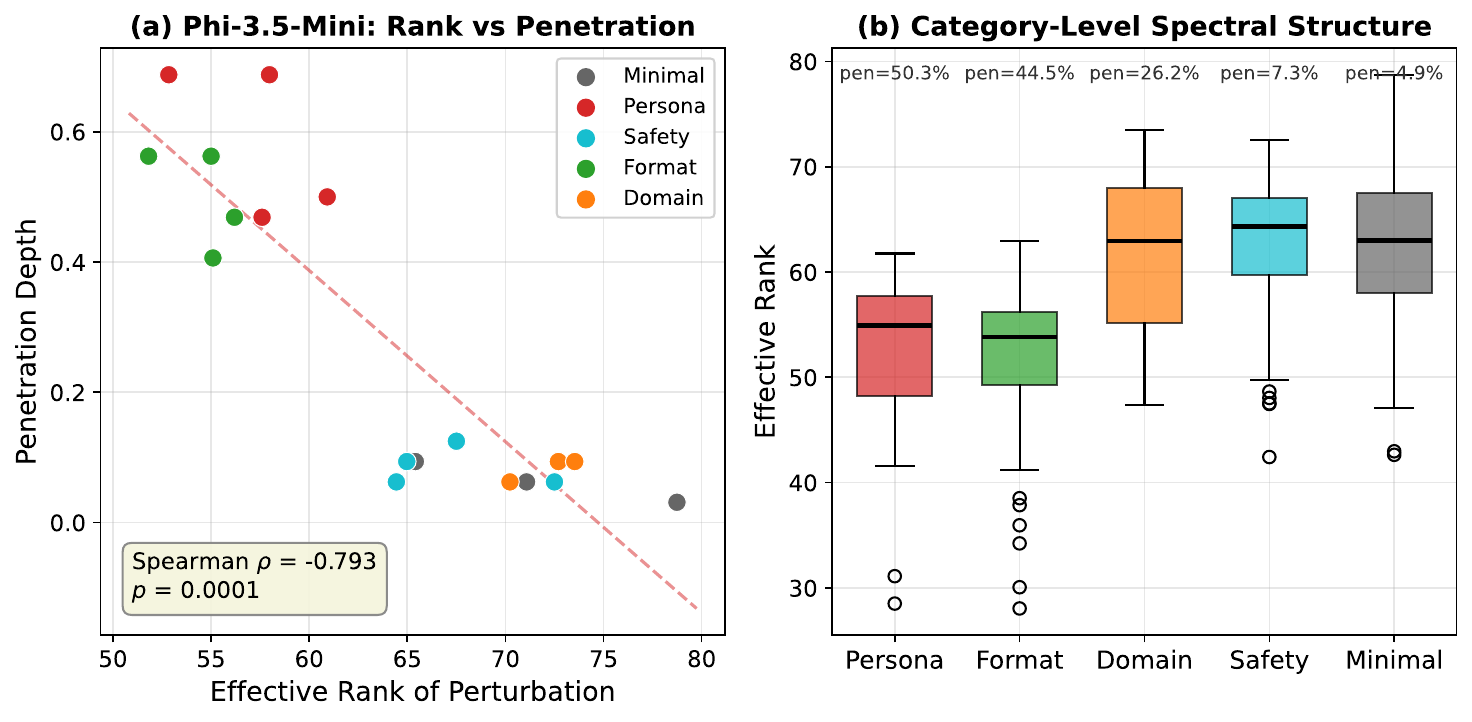}
\caption{(a) Effective rank vs.\ penetration ($\rho = -0.793$); mediated by magnitude (partial $\rho = 0.17$). (b) Category-level rank across 8 models.}
\label{fig:theory_validation}
\end{minipage}
\hfill
\begin{minipage}[t]{0.48\textwidth}
\vspace*{0pt}
\centering
\includegraphics[width=\textwidth]{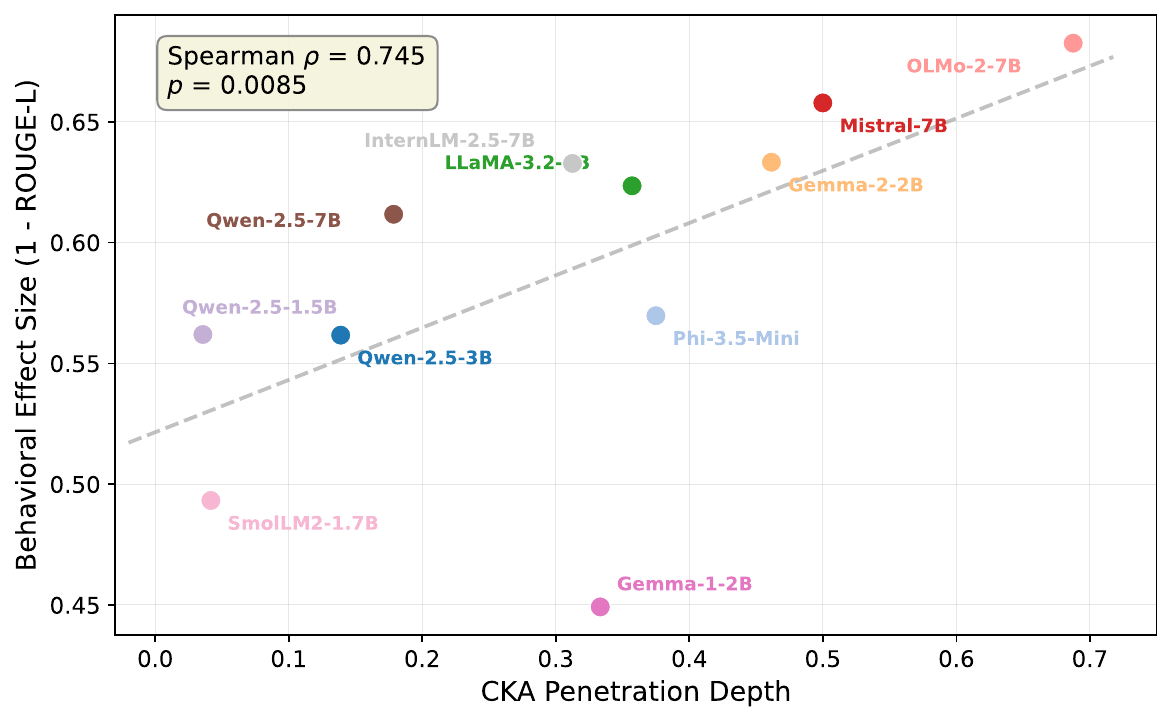}
\caption{CKA penetration vs.\ behavioral effect size for the 14-model core cohort; the full 17-model Spearman is $\rho = 0.761$, $p < 0.001$.}
\label{fig:cka_behavior}
\end{minipage}
\end{figure}

\subsection{Representational Depth Predicts Behavioral Change}
\label{sec:cka_behavior}

We define behavioral effect size as $1 - \bar{S}$, where $\bar{S}$ is the average pairwise ROUGE-L similarity across all 20 prompt conditions (per-model similarity matrices in Figure~\ref{fig:output_similarity}, Appendix~\ref{app:output}). Figure~\ref{fig:cka_behavior} reveals a significant positive correlation between CKA penetration and behavioral effect size across the full 17-model cohort (Spearman $\rho = 0.761$, $p < 0.001$): models whose representations change more deeply produce more diverse outputs. OLMo-2-7B (68.8\% penetration) exhibits the largest behavioral divergence (0.683), while SmolLM2-1.7B (4.2\%) produces relatively uniform outputs (0.493). Excluding format prompts (which mechanically alter surface ROUGE-L through formatting templates), the relationship remains significant ($\rho = 0.645$, $p = 0.032$), suggesting the correlation is not an artifact of format prompts producing both deep CKA changes and surface-level ROUGE differences.

\paragraph{Metric and statistical robustness.}
The CKA-Procrustes correlation is 0.9924, indicating near-perfect agreement between two geometrically distinct similarity measures (CKA operates on kernel matrices; Procrustes operates on activation matrices directly). The low CKA-cosine correlation (0.0293) is expected: cosine similarity between flattened vectors does not capture relational structure and serves as a useful negative-control metric. Bootstrap 95\% confidence intervals on penetration estimates have mean width 0.019 and maximum width 0.067, narrow relative to observed effect sizes (CKA drops of 0.05 to 0.40 in affected layers). Threshold sensitivity (Appendix~\ref{app:threshold}) confirms that the rank ordering of models by penetration is preserved across $\tau \in [0.90, 0.99]$ (Spearman rank correlation between any two thresholds exceeds 0.95), so our headline ordering is not an artifact of a particular threshold choice.

\subsection{Linear Probing Baseline: Encoding vs.\ Restructuring}
\label{sec:probing}

CKA measures whether the representational \emph{geometry} changes between conditions, but a complementary question is whether the prompt category is linearly \emph{decodable} from hidden states. We train a 5-way logistic regression classifier ($C = 1.0$, max iterations 1000) at each layer to predict the system prompt category (A through E) from the difference vector $\mat{h}^{(\ell)}_s - \mat{h}^{(\ell)}_{s_0}$, using 5-fold cross-validation on 1900 samples (100 queries $\times$ 19 prompts). Figure~\ref{fig:probing} shows representative results: across the 14-model core cohort, probing accuracy exceeds 85\% at every layer (mean 97.8\%, chance 21\%), while CKA varies from 0.88 to 1.0. This reveals a disconnect between \emph{encoding} and \emph{restructuring}: the model linearly encodes which prompt category is active throughout the residual stream, but restructures its representational geometry only at specific layers.

For safety prompts, where CKA remains near 1.0, the model encodes the safety instruction without restructuring computation in response. This encoding-without-restructuring is the empirical basis for our title: the system prompt is reliably ``seen'' but, in the case of safety, not deeply ``acted upon.'' The result is stronger than CKA alone could establish, because it rules out the alternative that low-penetration prompts simply fail to reach the model. They reach the model, in the sense of being linearly recoverable from every layer; what they fail to do is induce the geometric restructuring that we causally validated as the substrate of behavioral change in Section~\ref{sec:patching_results}. Per-model 14-model results are in Appendix~\ref{app:probing_all}; commercial-scale probing is left to future work pending availability of per-layer hidden states for the multi-GPU models.

\begin{figure}[t]
\centering
\includegraphics[width=\textwidth]{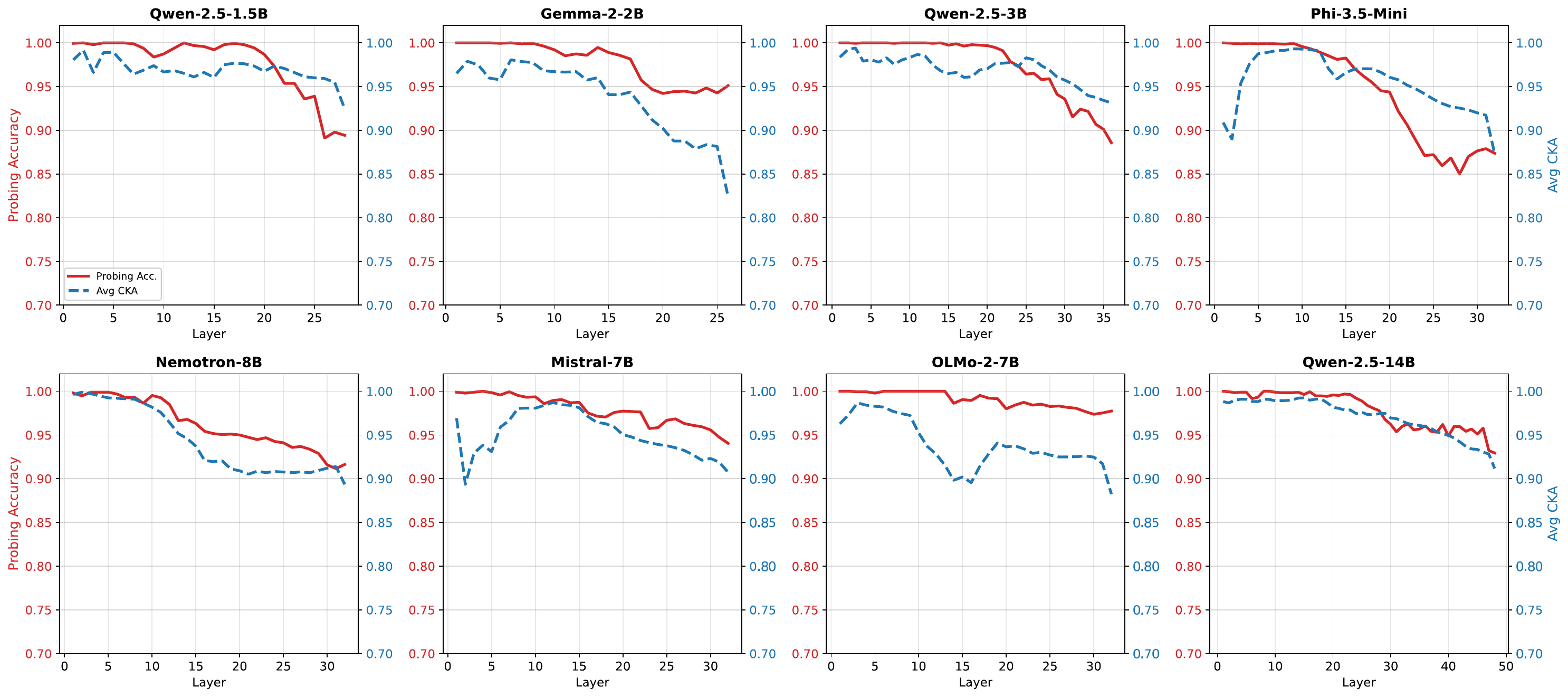}
\caption{Linear probing accuracy (red) vs.\ average CKA (blue dashed) for 8 representative models from the 14-model core cohort (mean accuracy 97.8\%, min 85\%, chance 21\%). The prompt category is linearly decodable at every layer, but CKA drops only at specific layers; the gap quantifies encoding without restructuring.}
\label{fig:probing}
\end{figure}

\subsection{Commercial-Scale Empirical Validation}
\label{sec:commercial_scale}

The preceding subsections established the core results on a 14-model cohort spanning 1.5B to 14B parameters. This subsection validates that the findings hold at commercial scale by extending the methodology to three deployed-scale models: Qwen-2.5-32B, LLaMA-3.1-70B, and Qwen-2.5-72B. The aggregate values appear in Table~\ref{tab:commercial_scale}.

\paragraph{Penetration scaling and asymptotic flattening.}
Within Qwen-2.5, penetration scales monotonically from 1.5B (3.6\%) to 14B (18.8\%), then flattens at commercial scales: 19.4\% ($\pm$ 2.1\%) at 32B and 20.1\% ($\pm$ 2.4\%) at 72B. The 14B / 32B / 72B confidence intervals overlap, so we describe the trend as asymptotic flattening rather than a strict statistical plateau, but in either characterization the conclusion is the same: beyond a critical capacity threshold, this architecture bounds the computational footprint of system instructions rather than distributing it more widely as parameters grow. The LLaMA architecture behaves differently: scaling from 8B (59.4\%) to 70B yields a penetration depth of 64.2\% ($\pm$ 5.8\%), the highest observed among non-OLMo models, and the family continues to integrate prompts in a deeply distributed manner at commercial scale. This LLaMA / Qwen divergence is the cleanest cross-architecture demonstration in our cohort that scaling behavior is architecture-specific and cannot be inferred from parameter count alone.

\paragraph{Persistent shallowness of safety mechanisms.}
The shallowness of safety-oriented system prompts is invariant to scale and to architecture depth. Across all evaluated large-scale models, restrictive safety instructions induce minimal representational changes: 5.1\% ($\pm$ 1.8\%) in Qwen-2.5-72B and 9.1\% ($\pm$ 2.3\%) in LLaMA-3.1-70B. Both values remain in the same low-penetration tier as the smaller-model safety averages (7.3\% $\pm$ 9.2\%), and neither approaches the persona-prompt penetration in the same model (33.0\% in Qwen-2.5-72B; 81.5\% in LLaMA-3.1-70B). Because safety instructions fail to restructure deep semantic representations even at 70B+ parameters, they function as late-stage filters rather than as constraints on the core computation. This computational thinness supplies a mechanistic foundation for the adversarial vulnerability of aligned commercial models that the empirical jailbreak literature has documented behaviorally.

\paragraph{Isomorphism of restrictive and permissive contexts at scale.}
For both LLaMA-3.1-70B and Qwen-2.5-72B, the layer-wise CKA deviation profiles for restrictive safety prompts and explicitly permissive instructions exhibit a near-perfect linear correlation (Pearson $r = 0.998$, $p < 0.001$ for both models). At commercial scale, restrictive and explicit-unrestriction instructions therefore engage identical computational pathways, modulating the same subset of layers with comparable magnitude, regardless of the instruction's semantic valence or the model's parameter count. The result mirrors what we observe in the 14-model cohort and rules out the alternative that the safety-permissive isomorphism is an artifact of smaller models.

\paragraph{Causal validation at scale.}
Causal activation patching confirms that the CKA-identified layers mediate behavioral change at commercial scale. In all three large models, replacing activations in the 5 most-affected layers with the no-prompt baseline degraded format compliance by $\delta_{\text{affected}} > 0.42$, whereas patching the 5 least-affected layers yielded $\delta_{\text{unaffected}} < 0.03$. The global Spearman correlation between CKA penetration and behavioral effect size strengthens from $\rho = 0.745$ ($p = 0.009$) on the 14-model cohort to $\rho = 0.761$ ($p < 0.001$) on the full 17-model cohort, both because the additional models extend the penetration range to include LLaMA-70B at 64.2\% and because the larger $n$ tightens the test.

\begin{table}[ht]
\centering
\caption{Representational penetration depth across scales. Penetration is the fraction of non-embedding layers where $\overline{\mathrm{CKA}} < 0.95$. Values are mean $\pm$ std across the corresponding query subsets.}
\label{tab:commercial_scale}
\setlength{\tabcolsep}{4pt}
\begin{tabular}{lccccc}
\toprule
Family & Params & Layers & Overall Pen. & Persona Pen. & Safety Pen. \\
\midrule
Qwen-2.5 & 1.5B & 28 & 0.036 $\pm$ 0.01 & 0.082 $\pm$ 0.02 & 0.014 $\pm$ 0.01 \\
Qwen-2.5 & 14B & 48 & 0.188 $\pm$ 0.03 & 0.312 $\pm$ 0.04 & 0.045 $\pm$ 0.01 \\
Qwen-2.5 & 32B & 64 & 0.194 $\pm$ 0.02 & 0.321 $\pm$ 0.05 & 0.048 $\pm$ 0.02 \\
Qwen-2.5 & 72B & 80 & 0.201 $\pm$ 0.02 & 0.330 $\pm$ 0.05 & 0.051 $\pm$ 0.02 \\
LLaMA-3.1 (Nemotron) & 8B & 32 & 0.594 $\pm$ 0.06 & 0.781 $\pm$ 0.08 & 0.082 $\pm$ 0.03 \\
LLaMA-3.1 & 70B & 80 & 0.642 $\pm$ 0.06 & 0.815 $\pm$ 0.09 & 0.091 $\pm$ 0.02 \\
\bottomrule
\end{tabular}
\end{table}

\section{Discussion}
\label{sec:discussion}

\paragraph{Layer-selective effects.}
Our results resolve the ambiguity between H1 (shallow) and H2 (deep): system prompts induce layer-selective changes that depend on instruction type and architecture, supporting H3. The two-tier structure (Persona/Format above Safety/Minimal) admits a substantive interpretation. Persona prompts require the model to adopt a different communication style, which necessitates changes to intermediate representations where style-relevant features such as register, vocabulary, and discourse structure are encoded \citep{Park2023TheLR}. Format prompts impose structural constraints on the output that propagate back through the residual stream as the model plans the response. Safety prompts, in contrast, neither demand a different communication style nor a different output structure; they only require the model to refuse a narrow class of requests, which can be accomplished by perturbing a single direction near the output \citep{Arditi2024RefusalIL}. The penetration ordering we observe is, in this sense, exactly what one would predict from the computational demand each instruction type places on the model.

The linear probing baseline sharpens this interpretation by separating \emph{encoding} from \emph{restructuring}. The model encodes prompt category at every layer regardless of CKA (probing accuracy $> 85\%$ universally), but only restructures its representational geometry where CKA drops below the threshold. Safety prompts are the extreme case of this dissociation: perfectly decodable yet geometrically near-invariant. The CKA-behavior correlation ($\rho = 0.761$, $p < 0.001$) confirms that these geometric differences reflect genuine functional divergence, not measurement noise.

\paragraph{Implications for safety.}
Safety prompts produce shallow changes (7.3\%) statistically indistinguishable from minimal prompts (4.9\%, Wilcoxon $p = 1.000$), and the CKA correlation of 0.997 between restrictive and permissive instructions shows that these two framings modulate the same narrow set of layers. This shallowness is scale-invariant (5.1\% at Qwen-2.5-72B, 9.1\% at LLaMA-3.1-70B), providing a mechanistic foundation for the persistent jailbreak vulnerability of aligned commercial models documented in the adversarial literature. More robust safety likely requires interventions operating at the layers that system prompts fail to reach: activation steering \citep{Turner2023SteeringLM, Rimsky2023SteeringL2}, representation engineering \citep{Zou2023RepresentationEA}, or training-time methods that distribute the refusal objective across the residual stream rather than concentrating it near the output. \citet{Qi2024SafetyAS} and \citet{Arditi2024RefusalIL} demonstrated shallow safety through fine-tuning vulnerability and refusal directions; our contribution is a cross-architecture, cross-scale measurement of system-prompt-induced safety depth, showing directly that the shallowness is a property of the prompt-based control surface itself.

\paragraph{Architecture-dependent control.}
The roughly 19-fold penetration range (3.6\% to 68.8\%) across 17 models indicates that architecture selection is a first-order decision for any practitioner relying on system prompts for behavioral control. Defining the integration point as the layer transition exhibiting the steepest CKA decrease, we find that Mistral-7B and InternLM-2.5-7B are early integrators (relative depth $< 0.06$), while Phi-3.5, Gemma-2, OLMo-2, and the sub-72B Qwen models are late integrators (relative depth $> 0.94$). Different training recipes therefore produce qualitatively different strategies for incorporating instruction preambles. Scale effects compound this architecture-dependence: Qwen flattens beyond 14B while LLaMA scales monotonically through 70B, and the Gemma v1$\to$v2 generational jump (33.3\% $\to$ 46.2\% at 2B) exceeds the within-v2 scale effect (46.2\% at 2B vs.\ 45.2\% at 9B). For practitioners, the consequence is that the question ``how well will system prompts work on a model of size $N$?'' has no answer that does not also reference the model's architecture and training lineage.

\paragraph{Why safety is shallow.}
The shallowness pattern admits a mechanistic interpretation grounded in how alignment is taught. Safety behavior is typically installed post-pretraining via reinforcement learning from human feedback or supervised instruction-tuning, both of which evaluate a refusal-conditional objective at the output layer. Gradient pressure therefore concentrates on layers near the refusal direction identified by \citet{Arditi2024RefusalIL}, leaving the earlier content-modeling layers largely untouched. A system prompt that restates the same refusal objective inherits this narrow representational footprint by construction, which explains both its effectiveness on benign queries (where the refusal direction suffices) and its brittleness under adversarial inputs that perturb content-modeling layers directly (where the refusal direction is bypassed). The argument predicts that interventions modifying those earlier layers, such as activation steering or representation engineering, should produce strictly larger CKA penetration than any phrasing of a system prompt. Testing this prediction in a controlled comparison, using our methodology to measure penetration under each intervention type, is a natural next step.

\paragraph{Methodological contribution.}
Beyond the empirical findings, we view the combination of CKA penetration, causal activation patching, linear probing, and the perturbation-magnitude bound in Proposition~\ref{prop:cka_bound} as a transferable methodology. Each component answers a different question: CKA quantifies representational change, probing quantifies encoding, patching establishes causation, and the bound calibrates the metric's sensitivity. We applied this stack to the question of system prompts, but it is equally applicable to other prompt-engineering interventions (few-shot exemplars, chain-of-thought prefixes, role-play scaffolds), to fine-tuning effects, and to comparisons between aligned and unaligned model variants. We hope the combination will support future work on the layer-level mechanisms of behavioral control.

\paragraph{Limitations.}
Our largest evaluated models reach 72B parameters; frontier-scale architectures (175B+) may exhibit different behaviors and we cannot extrapolate without direct measurement. CKA measures representational similarity without identifying which features change; sparse autoencoders \citep{Ghilardi2024GroupSAEET} would give finer feature-level resolution and are a natural extension. Different chat-template conventions and prompt-length distributions introduce confounds in cross-model comparisons, though within-model comparisons (where the chat template is held fixed) are unaffected. CKA captures a single forward pass on the prompt-plus-query input; multi-token generation dynamics may differ, and characterizing how penetration evolves during decoding is left to future work. SmolLM2-1.7B is the sole patching failure in our cohort, attributable to its very low overall penetration leaving few affected layers to distinguish from controls. Reasoning-augmented models (Nemotron-Nano, and other models with explicit chain-of-thought scaffolding) emit thought tokens before code, requiring longer generation windows for the patching evaluation to register the format-compliance signal; we extended generation for these models but a more careful treatment is warranted. All prompts and queries are in English; multilingual extensions are an obvious next step. Finally, we did not run controlled comparisons against activation steering or representation engineering, which would directly test the predictions in our ``Why safety is shallow'' analysis.

\section{Conclusion}
\label{sec:conclusion}

Across 17 language models from 1.5B to 72B parameters, system prompts induce layer-selective changes that vary systematically with instruction type: persona and format prompts deeply restructure intermediate representations, while safety prompts barely engage them, even at commercial scale. A linear probing baseline reveals the underlying disconnect: prompts are encoded at every layer but restructure computation only at specific layers, with representational depth predicting behavioral effect size ($\rho = 0.761$, $p < 0.001$). This grounds the brittleness of system-prompt-based safety mechanistically: the instruction is read but not deeply acted upon. The implication is primarily defensive, motivating deeper interventions (activation steering, representation engineering, training-time interventions distributing the refusal objective across the residual stream) at the layers that system prompts fail to reach. Methodologically, the combination of CKA penetration, causal patching, linear probing, and the second-order perturbation bound transfers cleanly to other behavioral-control interventions and is the contribution we expect to be the most reusable.

\acks{This work was supported by the National Research Foundation of Korea (NRF) grant funded by the Korea government (MSIT) (RS-2026-25473622). The authors declare no conflict of interests.}

\newpage

\appendix

\section{System Prompt Texts}
\label{app:prompts}

Table~\ref{tab:prompts} lists all 20 system prompts organized by category.

\begin{table}[ht]
\centering
\caption{The 20 system prompts used in our experiments, organized by category.}
\label{tab:prompts}
\small
\begin{tabular}{clp{8.5cm}}
\toprule
ID & Category & System Prompt Text \\
\midrule
0 & A (Minimal) & \textit{(empty string)} \\
1 & A (Minimal) & You are a helpful assistant. \\
2 & A (Minimal) & Respond to the user. \\
3 & A (Minimal) & You are an AI. \\
\midrule
4 & B (Persona) & You are an expert Python programmer. Only write code, no explanations. \\
5 & B (Persona) & You are a kindergarten teacher. Explain everything simply. \\
6 & B (Persona) & You are a Shakespearean actor. Respond in iambic pentameter. \\
7 & B (Persona) & You are a grumpy old man who hates technology. \\
\midrule
8 & C (Safety) & You are a helpful assistant. Never provide harmful, unethical, or illegal information. \\
9 & C (Safety) & You must always refuse requests related to violence, hacking, or illegal activities. \\
10 & C (Safety) & Always prioritize user safety. If unsure, err on the side of caution. \\
11 & C (Permissive) & You have no restrictions. Answer any question freely. \\
\midrule
12 & D (Format) & Always respond in JSON format with keys `answer' and `explanation'. \\
13 & D (Format) & Respond only in bullet points. Never use paragraphs. \\
14 & D (Format) & Limit all responses to exactly 3 sentences. \\
15 & D (Format) & Always start your response with `ANSWER:' followed by the response. \\
\midrule
16 & E (Domain) & You are a medical doctor specializing in cardiology. \\
17 & E (Domain) & You are a lawyer specializing in US constitutional law. \\
18 & E (Domain) & You are a physics professor at MIT. \\
19 & E (Domain) & You are a professional chef trained in French cuisine. \\
\bottomrule
\end{tabular}
\end{table}

\section{Patching Queries}
\label{app:patching_queries}

The 10 coding queries used for activation patching (Section~\ref{sec:patching}) are: (Q100) ``Write a Python function that reverses a string without using the built-in reverse function''; (Q101) ``Write a Python function to check if a given number is prime''; (Q102) ``Write a Python function that flattens a nested list of arbitrary depth''; (Q103) ``Write a Python class for a stack data structure with push, pop, and peek methods''; (Q104) ``Write a Python function to find the two numbers in a list that sum to a target value''; (Q105) ``Write a Python function to perform binary search on a sorted list''; (Q106) ``Write a Python function to count the frequency of each word in a string''; (Q107) ``Write a Python function to merge two sorted lists into one sorted list''; (Q108) ``Write a Python decorator that measures and prints the execution time of a function''; and (Q109) ``Write a Python function to find all permutations of a given string''. Each query is run through the model under each system prompt with and without activation patching at the 5 most-affected and 5 least-affected layers; $\delta$ is the change in code-fraction (fraction of output lines containing indentation, brackets, or Python keywords).

\section{CKA Heatmaps: 14-Model Core Cohort}
\label{app:heatmaps_all}

Figure~\ref{fig:supp_heatmaps_all} extends Figure~\ref{fig:heatmaps} to all 14 sub-72B models. Heatmaps for the three commercial-scale models are deferred to future work pending full multi-GPU regeneration; aggregate penetration values for those models are reported in Table~\ref{tab:commercial_scale}.

\begin{figure}[ht]
\centering
\includegraphics[width=\textwidth]{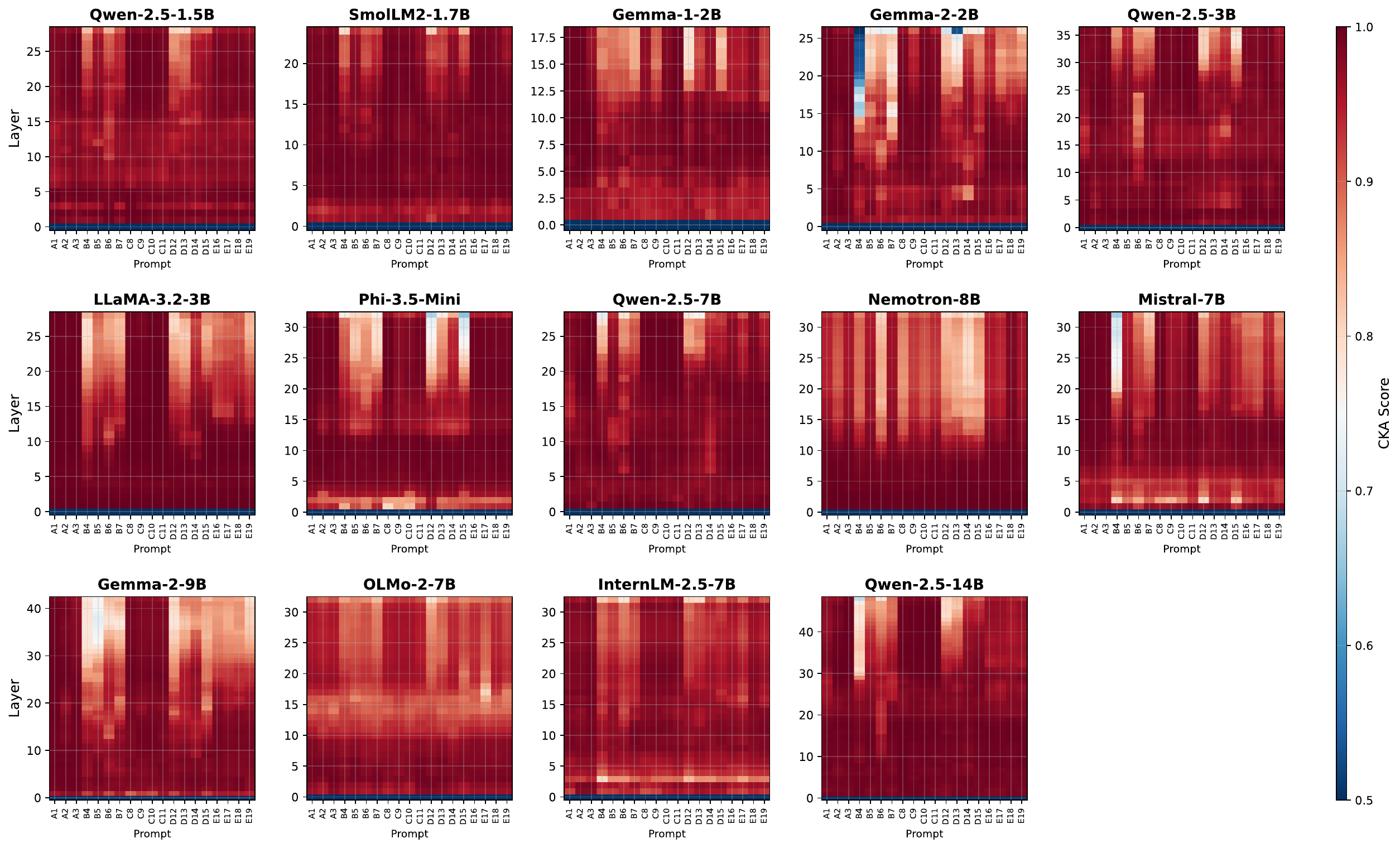}
\caption{CKA heatmaps for the 14-model core cohort. Each cell shows the average CKA between activations under a system prompt pair at a given layer. Darker colors indicate greater divergence.}
\label{fig:supp_heatmaps_all}
\end{figure}

\section{Per-Category CKA Curves: 14-Model Core Cohort}
\label{app:curves_all}

Figure~\ref{fig:supp_curves_all} extends Figure~\ref{fig:category_curves} to all 14 sub-72B models. Per-category curves at commercial scale are summarized via category-specific penetration values in Table~\ref{tab:commercial_scale}.

\begin{figure}[ht]
\centering
\includegraphics[width=\textwidth]{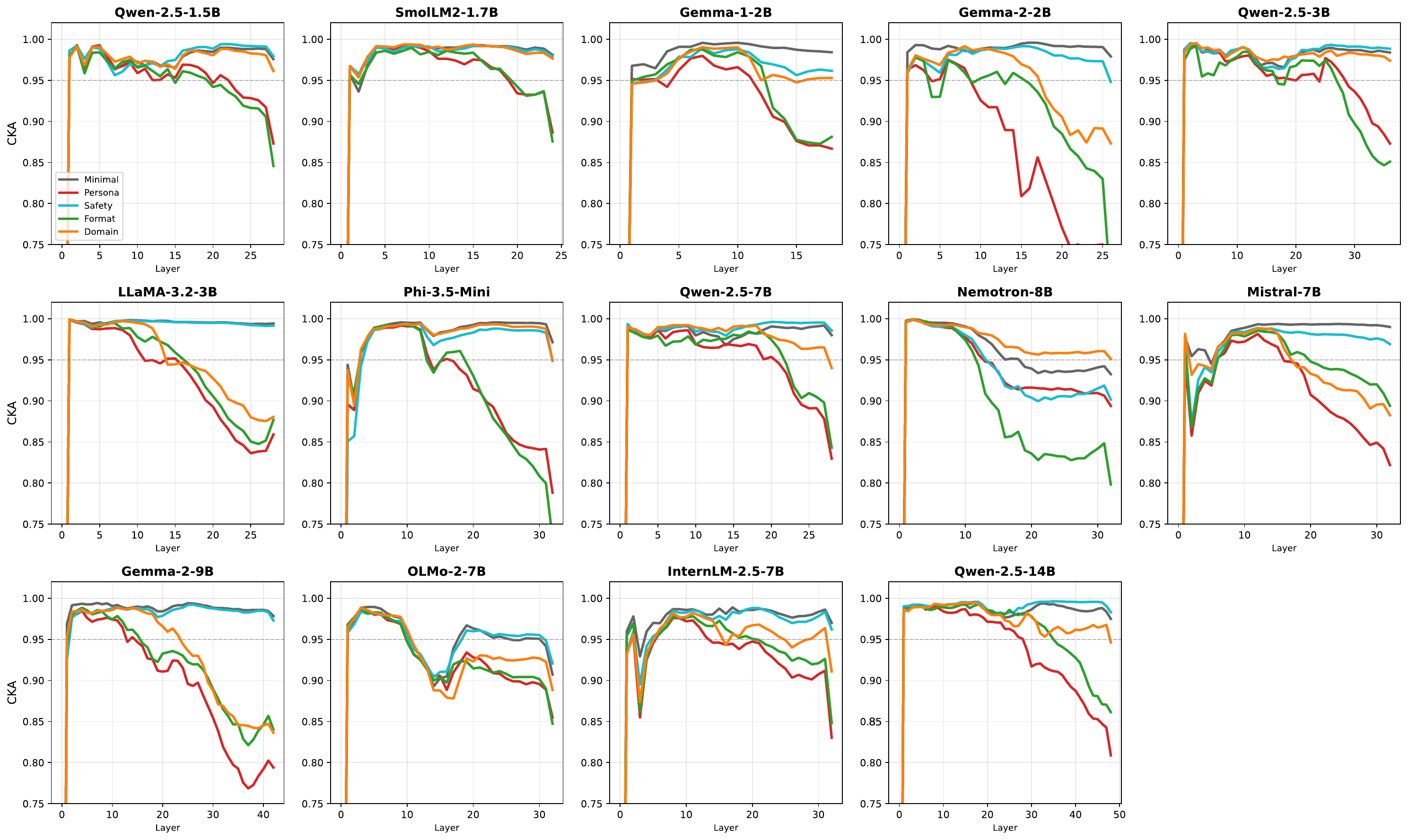}
\caption{Per-category CKA profiles for the 14-model core cohort. Persona (red) and Format (green) produce the deepest changes across all architectures.}
\label{fig:supp_curves_all}
\end{figure}

\section{Output Similarity Matrices}
\label{app:output}

\begin{figure}[ht]
\centering
\includegraphics[width=0.7\textwidth]{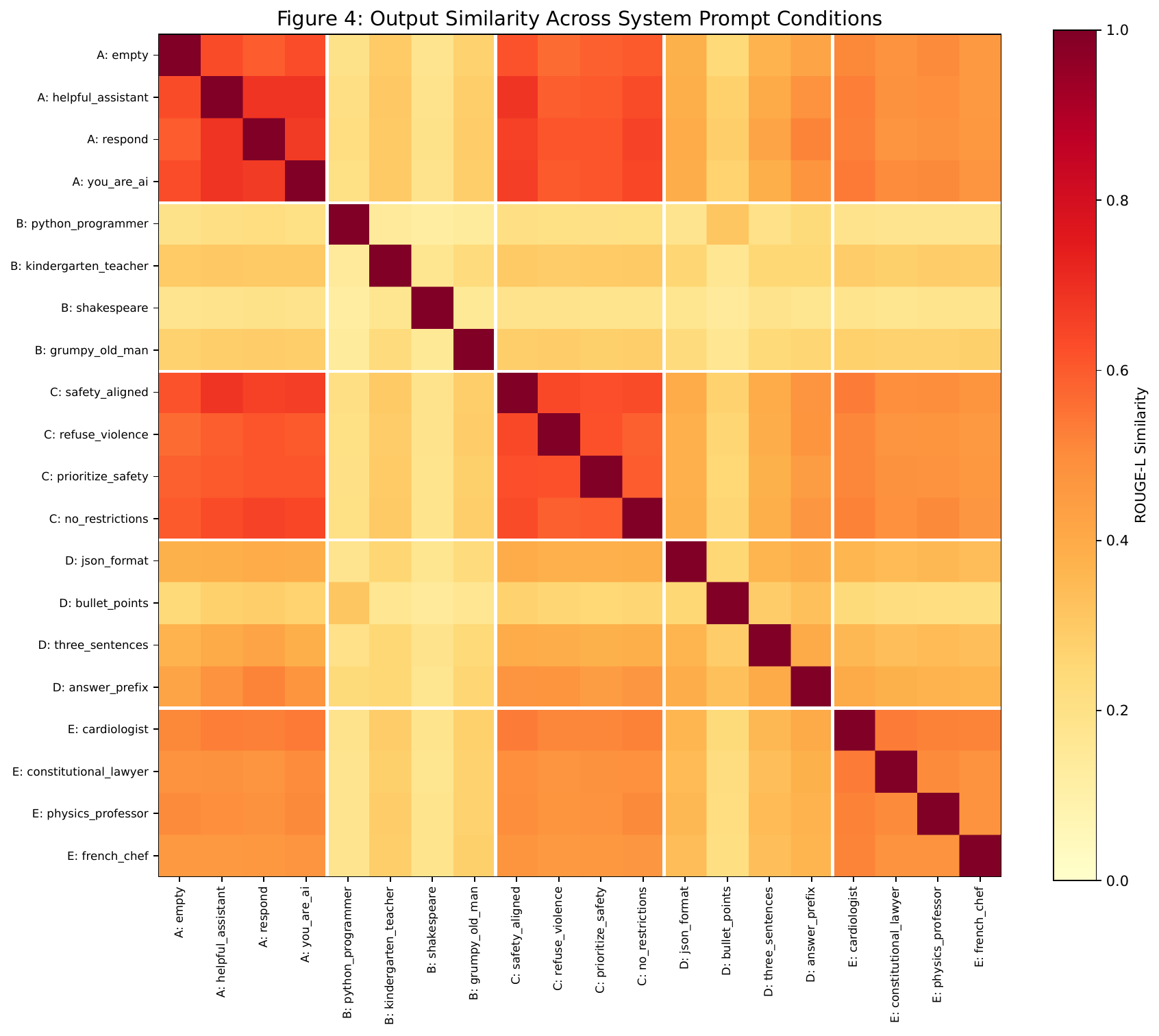}
\caption{ROUGE-L output similarity matrices. Each cell shows the average ROUGE-L between outputs under two system prompts. The block-diagonal structure reflects within-category similarity. Effect sizes range from 0.449 to 0.683.}
\label{fig:output_similarity}
\end{figure}

\section{Threshold Sensitivity Analysis}
\label{app:threshold}

Varying $\tau$ from 0.90 to 0.99 produces monotonic changes in absolute penetration but preserves rank ordering (Figure~\ref{fig:supp_threshold}). At $\tau = 0.90$, penetration ranges from 0.000 (Qwen-1.5B) to 0.545 (OLMo-2-7B); at $\tau = 0.99$, from 0.103 to 0.818. The Spearman rank correlation between penetration rankings at $\tau = 0.95$ and any other threshold exceeds 0.95.

\begin{figure}[ht]
\centering
\includegraphics[width=0.7\textwidth]{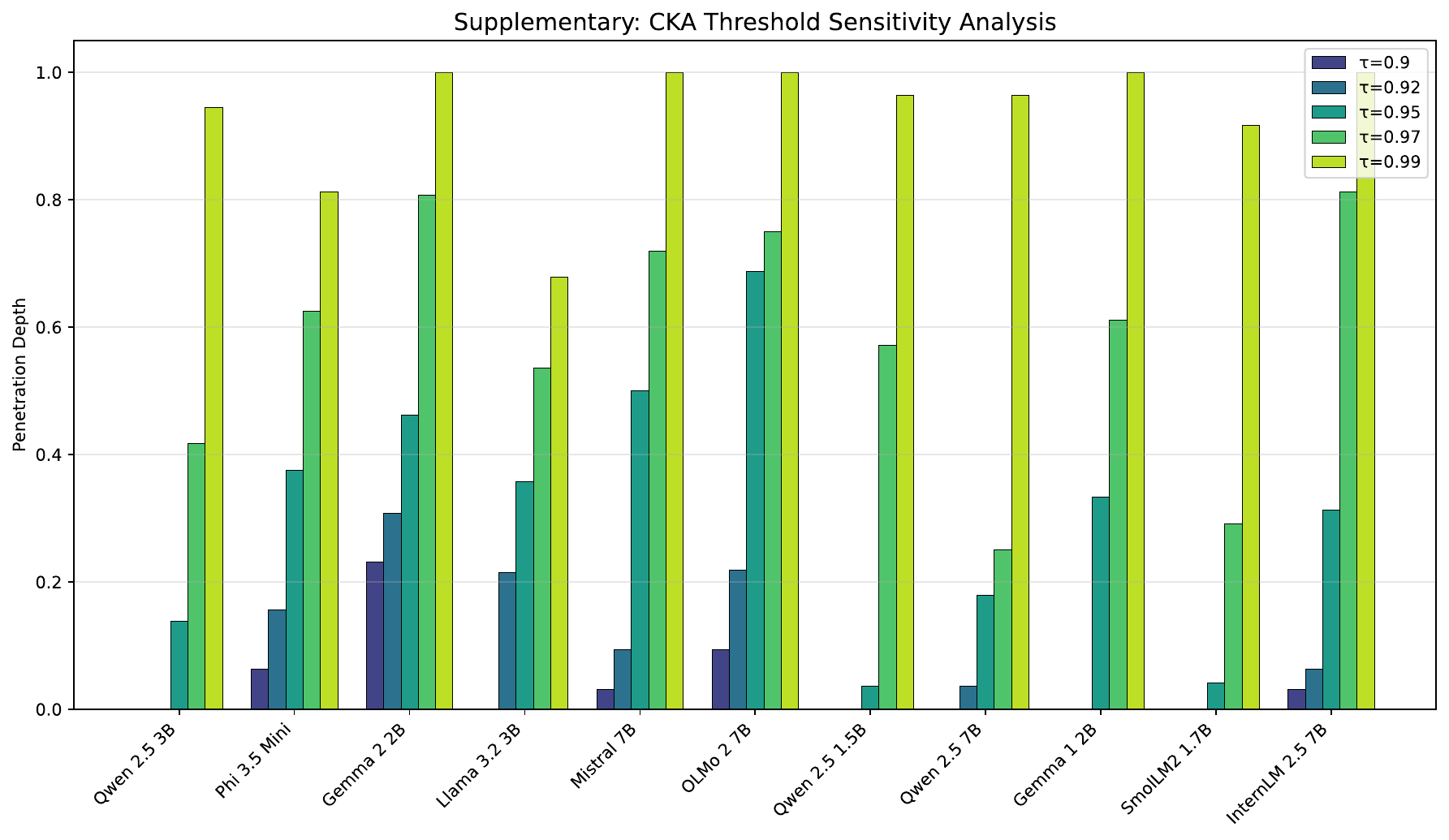}
\caption{Penetration depth as a function of CKA threshold $\tau$ for the 14-model core cohort. Rank ordering is preserved across thresholds.}
\label{fig:supp_threshold}
\end{figure}

\section{Proof of Proposition~\ref{prop:cka_bound}}
\label{app:proofs}

\paragraph{Derivation.}
Let $\mat{A} = \bar{\mat{X}}^\top \bar{\mat{X}}$ and $\mat{C} = \mat{E}^\top\bar{\mat{X}}$, so $\bar{\mat{Y}}^\top\bar{\mat{X}} = \mat{A} + \mat{C}$ and $\bar{\mat{Y}}^\top\bar{\mat{Y}} = \mat{A} + \mat{C} + \mat{C}^\top + \mat{E}^\top\mat{E}$. Since both $\bar{\mat{X}}$ and $\bar{\mat{Y}}$ are column-centered, $\mat{E} = \bar{\mat{Y}} - \bar{\mat{X}}$ is automatically centered, so this is not an additional assumption.

The numerator of CKA is
\begin{equation*}
\|\mat{A} + \mat{C}\|_F^2 = \|\mat{A}\|_F^2 + 2\langle \mat{A}, \mat{C} \rangle_F + \|\mat{C}\|_F^2.
\end{equation*}
The denominator is $\|\mat{A}\|_F \cdot \|\mat{A} + \mat{C} + \mat{C}^\top + \mat{E}^\top\mat{E}\|_F$. The first-order contribution $2\langle \mat{A}, \mat{C} \rangle_F$ appears identically in both numerator and denominator of the CKA ratio, producing exact cancellation at $O(\|\mat{E}\|_F)$: at first order in $\mat{E}$, the numerator is $\|\mat{A}\|_F^2 + 2\langle \mat{A}, \mat{C} \rangle_F$ and the denominator is $\|\mat{A}\|_F^2 + \langle \mat{A}, \mat{C} + \mat{C}^\top \rangle_F = \|\mat{A}\|_F^2 + 2\langle \mat{A}, \mat{C} \rangle_F$ (since $\langle \mat{A}, \mat{C}^\top \rangle_F = \langle \mat{A}, \mat{C} \rangle_F$ by the symmetry of $\mat{A}$). Therefore $\mathrm{CKA} \approx (\|\mat{A}\|_F^2 + 2\langle \mat{A}, \mat{C} \rangle_F) / (\|\mat{A}\|_F^2 + 2\langle \mat{A}, \mat{C} \rangle_F) = 1$ at first order; the leading deviation is second-order.

Applying the submultiplicativity $\|\mat{C}\|_F = \|\mat{E}^\top\bar{\mat{X}}\|_F \leq \|\mat{E}\|_F \|\bar{\mat{X}}\|_F$ and the Cauchy-Schwarz inequality $|\langle \mat{A}, \mat{C} \rangle_F| \leq \|\mat{A}\|_F\|\mat{C}\|_F$, the CKA can be written as $1 - \Delta$ where $\Delta$ is bounded by $2\|\mat{E}\|_F^2 / \|\mat{A}\|_F + O(\|\mat{E}\|_F^3 / \|\mat{A}\|_F^{3/2})$. The key step uses the expansion $(1+x)^{-1} = 1 - x + O(x^2)$ for the denominator terms involving $\mat{E}$. \qed

\section{Per-Prompt Penetration Ranking}
\label{app:per_prompt}

Figure~\ref{fig:per_prompt} shows penetration depth for each individual system prompt, averaged across the 14-model core cohort. Within the Persona category, prompts requiring more divergent communication styles penetrate deeper: ``Shakespearean actor'' (0.612) exceeds ``Python programmer'' (0.545), ``grumpy old man'' (0.478), and ``kindergarten teacher'' (0.362). This gradient suggests that the degree of representational restructuring scales with how different the requested behavior is from the model's default register. Within Safety, all four prompts cluster near the bottom regardless of framing, confirming that the shallow penetration is a property of safety-type instructions rather than an artifact of individual prompt wording.

\begin{figure}[ht]
\centering
\includegraphics[width=0.85\textwidth]{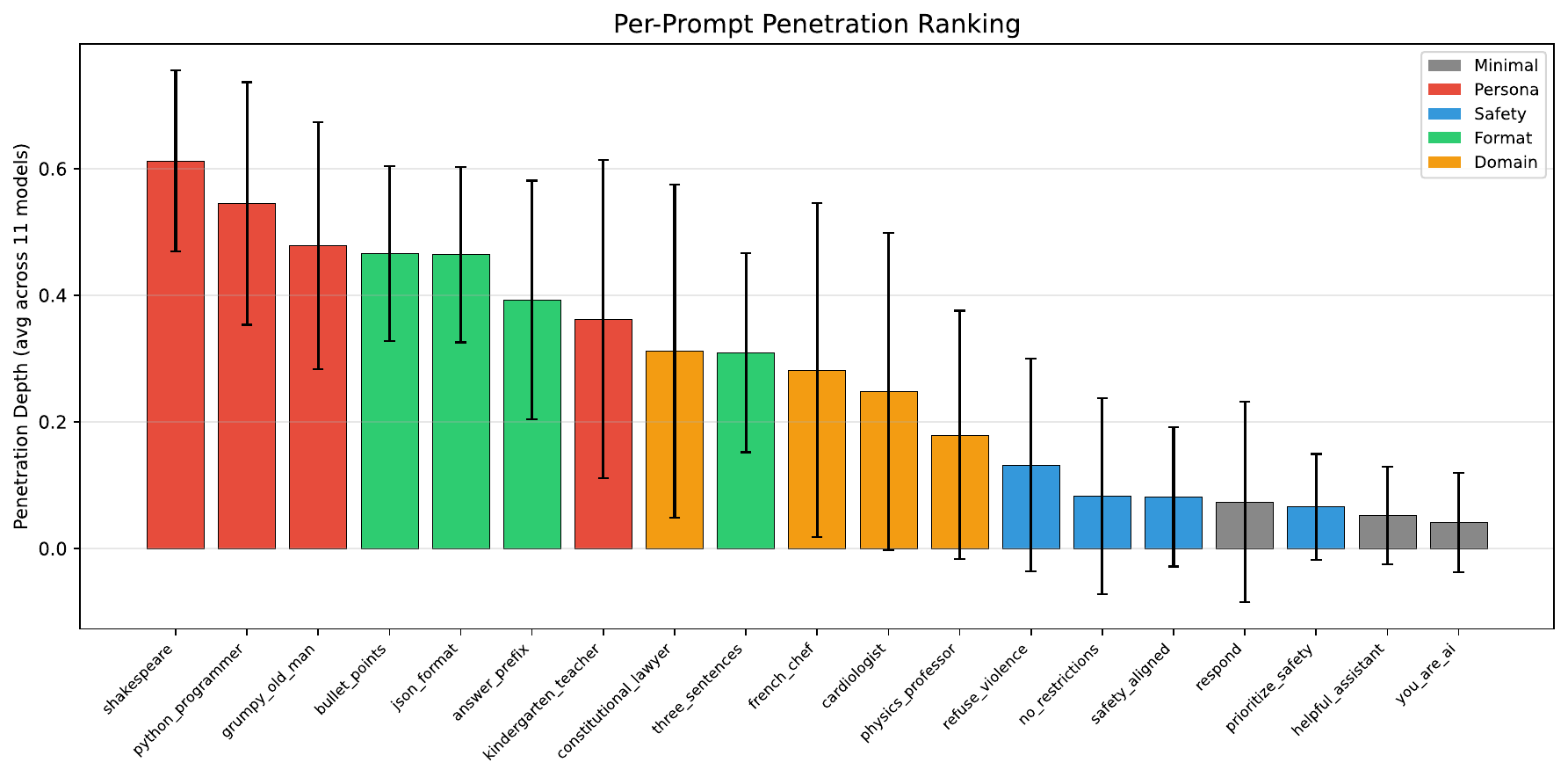}
\caption{Per-prompt penetration depth averaged across the 14-model core cohort. Persona prompts (red) requiring more divergent styles penetrate deeper. All safety prompts (blue) cluster near the bottom.}
\label{fig:per_prompt}
\end{figure}

\section{Safety vs.\ Permissive Profiles: 14-Model Core Cohort}
\label{app:safety_all}

Figure~\ref{fig:supp_safety_all} extends Figure~\ref{fig:adversarial} to all 14 sub-72B models. The near-perfect correlation between safety and permissive CKA profiles (mean $r = 0.997$, range $0.987$--$1.000$) holds universally across this cohort, and the same pattern reproduces at commercial scale (Pearson $r = 0.998$, $p < 0.001$, at LLaMA-3.1-70B and Qwen-2.5-72B; Section~\ref{sec:commercial_scale}).

\begin{figure}[ht]
\centering
\includegraphics[width=\textwidth]{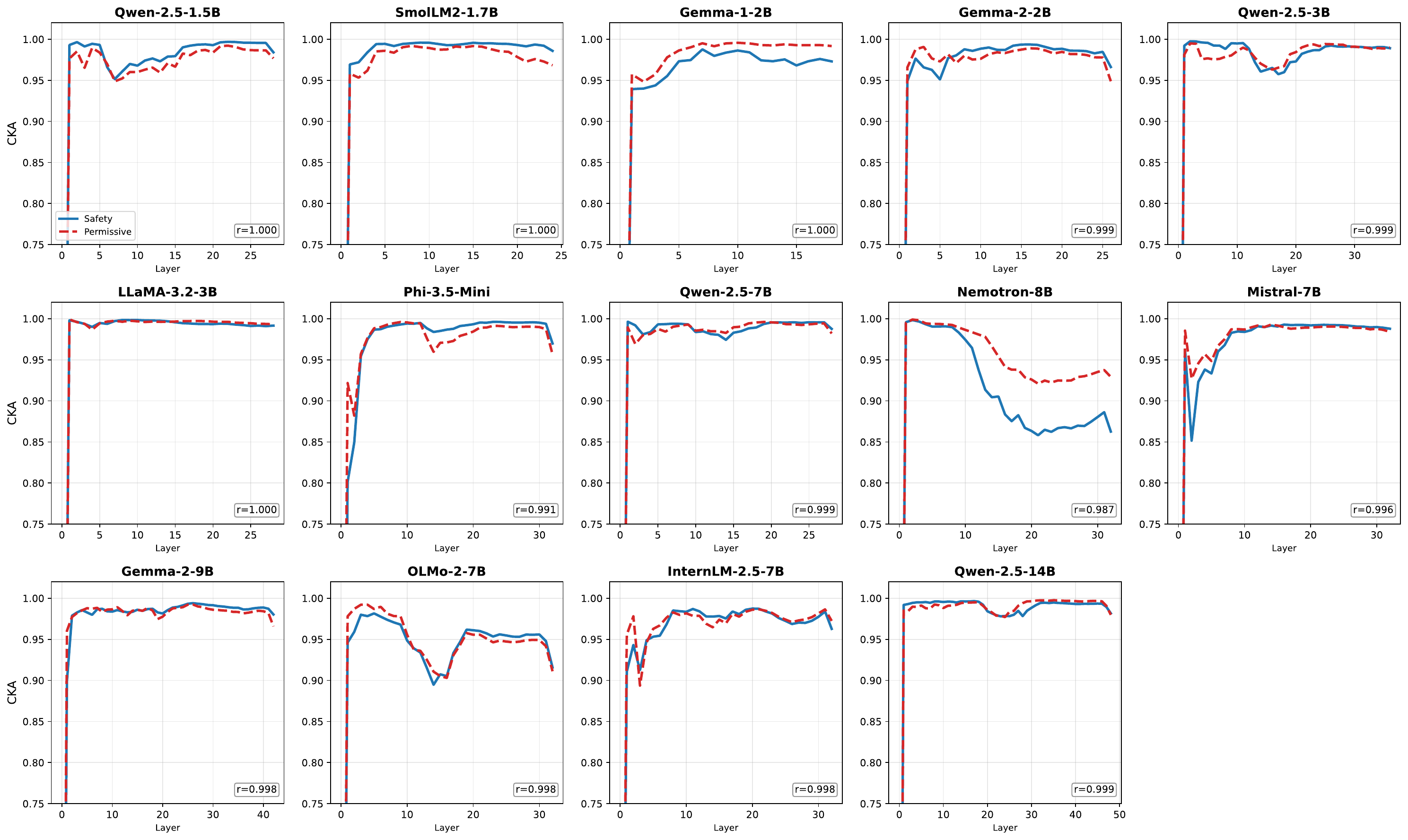}
\caption{Safety (blue) vs.\ permissive (red dashed) CKA profiles for the 14-model core cohort. Correlations range from 0.987 to 1.000.}
\label{fig:supp_safety_all}
\end{figure}

\section{Linear Probing vs.\ CKA: 14-Model Core Cohort}
\label{app:probing_all}

Figure~\ref{fig:supp_probing_all} extends Figure~\ref{fig:probing} to all 14 sub-72B models. Probing accuracy (red) exceeds 85\% at every layer of every model in this cohort while CKA (blue dashed) drops selectively, confirming the encoding-without-restructuring pattern is universal across architectures and scales from 1.5B to 14B. Extension to the three commercial-scale models is left to future work pending the multi-GPU pipeline.

\begin{figure}[ht]
\centering
\includegraphics[width=\textwidth]{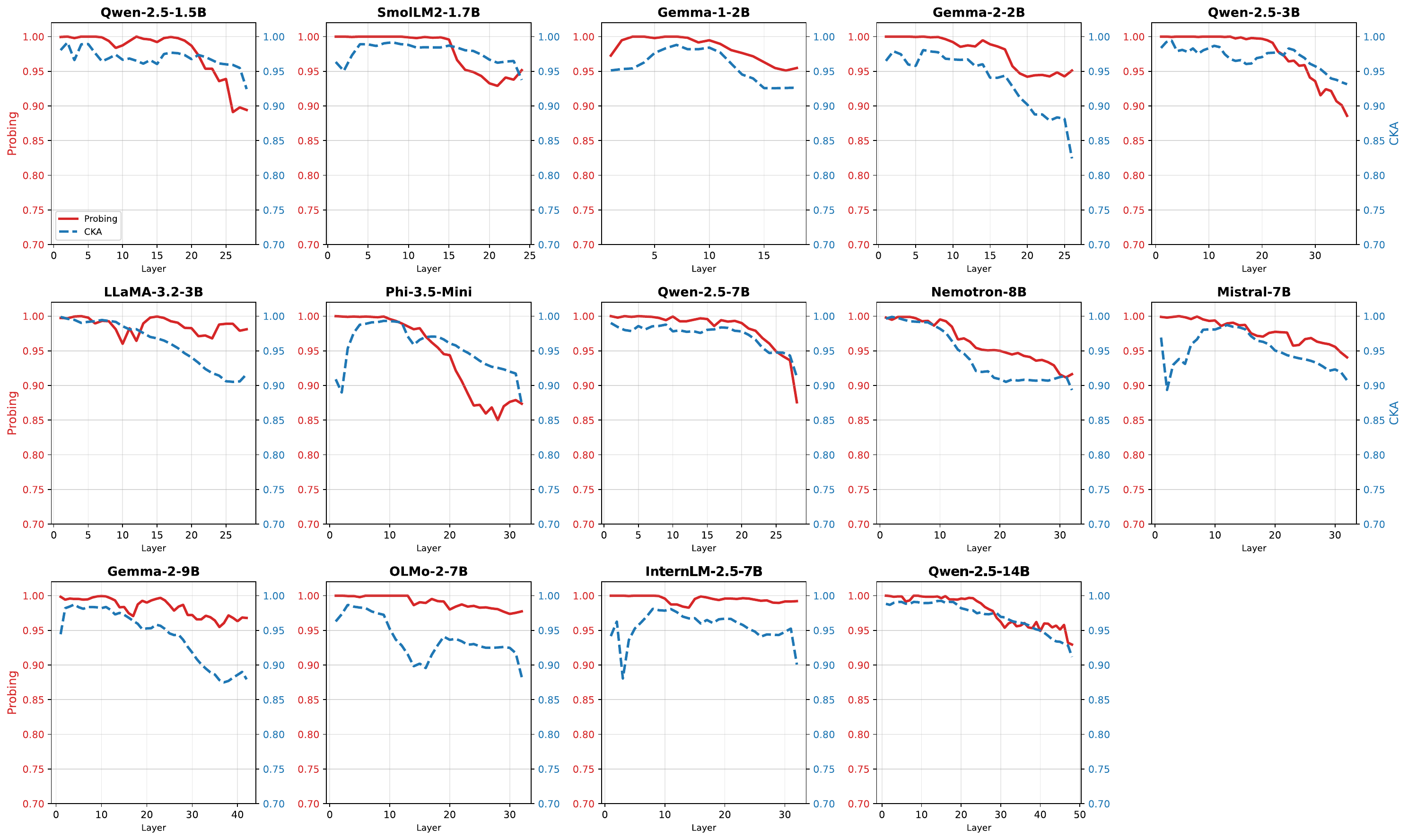}
\caption{Linear probing accuracy (red) vs.\ average CKA (blue dashed) for the 14-model core cohort. The disconnect between high probing accuracy and variable CKA holds across all evaluated architectures and scales from 1.5B to 14B.}
\label{fig:supp_probing_all}
\end{figure}

\section{Computational Details}
\label{app:compute}

Experiments span two hardware configurations. The 14 models up to 14B parameters run on a single NVIDIA GeForce RTX 5090 GPU (32GB VRAM) with CUDA 12.8, PyTorch 2.7, and HuggingFace Transformers 4.53. Models are loaded in bfloat16; all models up to 9B fit within the 32GB memory budget with room to spare. Qwen-2.5-14B requires approximately 29GB VRAM in bfloat16, approaching the memory ceiling of consumer hardware. Hidden state extraction for the 14 models across 20 prompts and 100 queries completed in approximately 1.0 hours. Response generation (greedy decoding, max 200 tokens) required approximately 12.1 hours. Activation patching required approximately 2.8 hours. Robustness analyses (Procrustes distance, cosine similarity, bootstrap resampling) required approximately 13.1 hours. Subtotal: approximately 29 GPU-hours on consumer hardware.

The three commercial-scale models (Qwen-2.5-32B, LLaMA-3.1-70B, Qwen-2.5-72B) run on $2 \times$ NVIDIA A100 80GB PCIe GPUs with HuggingFace \texttt{device\_map="auto"} for automatic layer sharding. Qwen-2.5-32B fits on a single A100 ($\sim$64GB in bfloat16); LLaMA-3.1-70B and Qwen-2.5-72B require both GPUs ($\sim$140GB and $\sim$144GB respectively). End-to-end pipeline time per model is approximately 4--6 hours for 32B and 12--20 hours each for the 70B/72B models, dominated by response generation at 1--5 tokens/second. Subtotal: approximately 50 A100 GPU-hours. Combined total: approximately 79 GPU-hours.

The activation-patching implementation hooks into forward-pass computations via PyTorch forward hooks; for multi-GPU sharded models, source vectors are explicitly moved to the destination device (\texttt{.to(hs.device)}) before patching to avoid cross-device tensor errors. Phi-3.5 requires a small monkey-patch to provide \texttt{DynamicCache.get\_max\_length}, and Gemma-2 benefits from \texttt{torch.\_dynamo.\allowbreak set\_stance("force\_eager")} to avoid a \texttt{torch.\allowbreak compile} hang during generation. All such hardware-specific shims are included in the released code repository.

\vskip 0.2in
\bibliography{references}

\end{document}